\documentclass[conference]{IEEEtran}
\IEEEoverridecommandlockouts

\usepackage{microtype}
\usepackage{graphicx}
\usepackage{booktabs}
\usepackage{amsmath,amssymb,amsfonts}
\usepackage{amsthm}
\usepackage[square,numbers,sort&compress]{natbib}
\usepackage[hidelinks,pdfauthor={Xudong Wang, Chris Ding, Tongxin Li, Jicong Fan},pdftitle={GLASS: Graph-Language Alignment with Spherical Scoring for Transferable Graph-Level Anomaly Detection},pdfkeywords={graph-level anomaly detection, graph-language alignment, multimodal learning, cross-domain transfer, differential privacy}]{hyperref}
\usepackage{enumitem}
\usepackage{multirow}
\usepackage{xspace}
\usepackage{placeins}
\usepackage{cuted}
\usepackage{xcolor}

\usepackage{bm}
\providecommand{\uline}[1]{\underline{#1}} % Underlined table entries; no ulem dependency.

\setlist[itemize]{leftmargin=*,topsep=1pt,itemsep=1pt,parsep=0pt}

\theoremstyle{plain}
\newtheorem{proposition}{Proposition}

\newcommand{\ouralg}{\textsf{GLASS}\xspace}

\newcommand{\std}[2]{\ensuremath{#1_{\pm #2}}}

\newcommand{\best}[1]{{\bfseries\boldmath #1}}
\newcommand{\second}[1]{\uline{#1}}
\newcommand{\third}[1]{{\itshape #1}}

\title{GLASS: Graph-Language Alignment with Spherical Scoring for Transferable Graph-Level Anomaly Detection}

\author{
\IEEEauthorblockN{Xudong Wang, Chris Ding, Tongxin Li, Jicong Fan\textsuperscript{*}}
\IEEEauthorblockA{
School of Data Science,\\
The Chinese University of Hong Kong, Shenzhen (CUHK-Shenzhen), China\\
xudongwang@link.cuhk.edu.cn, \{chrisding, litongxin, fanjicong\}@cuhk.edu.cn
}
\thanks{\textsuperscript{*}Corresponding author: Jicong Fan.}
}

\begin{document}
\maketitle

\begin{abstract}
We introduce \ouralg, a framework for graph-level anomaly detection (GLAD) that achieves robust cross-domain transferability through graph-language alignment on the unit hypersphere. \ouralg builds a unified representation space by aligning a structure-aware graph encoder with an instruction-aware text embedding via a multi-slice soft cosine objective. Our framework serializes local, global, and semantic graph properties into a compact Graph Descriptor Prompt (GraphDP), creating a text bridge that enables domain-agnostic anomaly scoring. By enforcing multi-scale consistency through Matryoshka representation slices, the model captures anomalous deviations at multiple levels of granularity. We formulate anomaly detection as density estimation on the aligned hypersphere and introduce Spherical Multi-Modal Scoring (SMS), which instantiates von Mises-Fisher kernel density estimators in both graph and text embedding spaces. This probabilistic formulation recovers angular 1-nearest-neighbor scoring in the high-concentration limit, motivates the practical mean $k$-nearest-neighbor scorer, and provides a principled fusion of structural and semantic anomaly signals. The shared text embedding space further serves as a cross-domain bridge: by encoding a target domain's GraphDP without target-domain training data, \ouralg performs zero-shot anomaly detection, and with only a handful of normal examples, few-shot adaptation via reference-set calibration. For privacy-sensitive deployment, we extend reference-set calibration with a bounded joint graph-text kernel summary that provides graph-record differential privacy while keeping the encoders fixed independently of the private target references. Across twelve benchmarks and three meta-domains, \ouralg obtains the best average AUROC and rank compared with recent advanced GLAD baselines and enables effective cross-domain transfer.
\end{abstract}


\section{Introduction}
Graphs are primary data objects in molecular screening, protein analysis, social computing, and many other data-mining applications. In graph-level anomaly detection (GLAD), a model observes normal graph instances and must rank unseen graphs by anomalousness~\citep{akoglu2015graph,ma2021comprehensive,qiao2024deep}. Existing GLAD systems usually learn a detector for one dataset: molecules have atom and bond attributes, protein graphs have different structural attributes, and social graphs may have no node attributes at all. This dataset-specific assumption becomes restrictive when a new domain has only a few trusted normal graphs, or no target-domain training data.

The central difficulty is not merely discriminative capacity; it is the absence of an interoperable representation and a common anomaly-scoring principle across heterogeneous graph families. A molecular ring, a protein contact pattern, and a social community are not directly comparable as raw node features. However, they can all be described through structural language: graph size, density, degree profiles, clustering, motifs, connectivity, core structure, and spectral summaries. This observation suggests a scalable route toward transferable GLAD: align graph representations with an instruction-aware language embedding space that can encode the same structural evidence across domains.

\ouralg\footnote{Code: \url{https://github.com/MathAdventurer/GLASS}.} instantiates this idea through a coherent pipeline. First, Local Topology Descriptors (LTDs) provide stable node-level structural evidence even when raw attributes are missing or incomparable. Second, GraphDP serializes local, global, and spectral graph properties into a compact structured prompt. Third, a frozen instruction-aware text encoder embeds this prompt and provides stable text anchors; we use Qwen3-Embedding~\citep{qwen3embedding}, an instruction-aware embedding model with native MRL support, as the embedding function rather than a generative reasoner. Finally, a structure-aware graph encoder is trained to align with these anchors on the unit hypersphere through native Matryoshka representation slices.

The hyperspherical view turns GLAD into density estimation. In the aligned graph--text space, a query graph is anomalous when it lies in a low-density region relative to normal references. We formalize this through a von Mises--Fisher (vMF) kernel-density view; its exact high-concentration limit is angular 1-NN, which motivates a robust mean $k$-NN implementation. This scoring family supports three deployment regimes without changing the architecture: single-domain GLAD uses target normal references; zero-shot transfer uses source-domain normal references; and few-shot adaptation augments the reference set with trusted target normal graphs.

Reference availability also has a privacy dimension: an institution may possess normal graphs but be unable to share their individual representations. This motivates a private-reference extension rather than private fine-tuning of the entire text model. DP-MERF~\citep{harder2021dpmerf} and efficient private KDE~\citep{wagner2023privatekde} establish one-shot kernel-summary mechanisms. We exploit the equivalence between a bounded vMF similarity kernel and a Gaussian kernel restricted to the sphere to privatize a joint graph--text reference summary. The encoders remain public and frozen; only the summary is released. This extension preserves the representation and density principle, but replaces exact reference-neighbor scoring with a finite-bandwidth approximation whose privacy and approximation costs are evaluated separately.

This design addresses three impediments to transferable GLAD. First, it replaces dataset-specific raw features with a structural-language bridge. Second, it stabilizes cross-domain alignment by freezing the instruction-aware text space and learning graph-side mappings only. Third, it provides a nonparametric spherical scoring rule that is naturally compatible with reference-set calibration. Empirically, \ouralg achieves strong single-domain performance across twelve benchmarks and reveals interpretable transfer regimes: molecular and protein graphs transfer well because their GraphDP descriptions share substructure semantics, while larger biological-to-social shifts require richer source weighting or target references.

Our contributions are:
\begin{itemize}
    \item We propose a transferable GLAD paradigm that formulates anomaly detection as density estimation on an aligned graph--language hypersphere.
    \item We instantiate this paradigm in \ouralg, integrating LTDs, GraphDP, frozen instruction-aware text anchors, native Matryoshka slices, and Spherical Multi-Modal Scoring.
    \item We give a rigorous vMF KDE analysis showing why angular nearest-neighbor scoring arises as the high-concentration density limit.
    \item Our empirical results demonstrate \ouralg's superior single-domain performance and effective zero-/few-shot cross-domain transfer across twelve widely used graph benchmarks and three meta-domains.
    \item We extend reference calibration to graph-record differential privacy using a bounded joint kernel summary, and evaluate privacy budgets, reference sizes, modalities, and one-round distributed-summary noise with independently trained public-source encoders.
\end{itemize}

\section{Related Work}
\textbf{Graph-level anomaly and graph OOD detection.}
Graph anomaly detection has been studied at node, edge, subgraph, and graph levels~\citep{akoglu2015graph,ma2021comprehensive,qiao2024deep,wang2025learnable}. GLAD differs from node-level anomaly detection because the object being ranked is an entire graph and the normality signal may lie in global topology, local motifs, node features, or their interaction. Class-based GLAD and OOD protocols emphasize consistent normal-only training and held-out anomaly evaluation~\citep{liu2023good,wang2025unifying}. Deep one-class methods adapt hypersphere learning to graph encoders, as in OCGIN and OCGTL~\citep{qiu2022raising,zhao2023using}. GLocalKD uses global-local knowledge distillation~\citep{ma2022deep}; GLADC and CVTGAD use contrastive or cross-view objectives~\citep{luo2022deepgladc,li2023cvtgad}; SIGNET introduces a subgraph information bottleneck~\citep{liu2023signet}. Recent reconstruction and unified anomaly frameworks such as MUSE and UniFORM further improve graph-level baselines~\citep{kim2024muse,song2025uniform}. These methods are strong single-domain detectors, but their standard protocol assumes normal target-domain graphs.

\textbf{Structural features, graph kernels, and graph language.}
Classical graph kernels remain competitive in GLAD because structural statistics are robust when node attributes are missing or incomparable. We include propagation-kernel and Weisfeiler-Lehman graph kernels with one-class SVM and isolation forest detectors~\citep{neumann2016propagation,shervashidze2011weisfeiler,liu2008isolation}. Local degree-profile and local topological-profile work shows that local statistics can be strong graph representations~\citep{cai2018ldp,adamczyk2023ltp}. \ouralg uses these statistics differently: they are shared evidence for graph encoding and GraphDP serialization. Graph-language methods study graph serialization, graph prompting, and alignment between GNNs and language embeddings~\citep{fatemi2024talk,perozzi2024graphtoken,wang2024llms,ye2024instructglm}; GraphMAE2 and OFA improve generalization for graph representation learning~\citep{hou2023graphmae2,liu2024ofa}; AnomalyGFM and ARC explore zero-/few-shot anomaly detection mainly in node-level settings~\citep{liu2024arc,qiao2025anomalygfm}. \ouralg instead uses frozen text embeddings as structural anchors for unsupervised graph-level anomaly scoring.

\textbf{Spherical and density-based scoring.}
Hyperspherical representations reduce sensitivity to feature norms and support angular density models~\citep{banerjee2005vmf,ming2023cider,bai2024hypo}. Kernel density estimation and nearest-neighbor anomaly scoring are classical nonparametric tools. On the sphere, vMF kernels provide a natural directional density family~\citep{pelletier2005kernel,garcia2013exact}. \ouralg connects this statistical view to graph--language alignment: the text bridge supplies a transferable coordinate system, while spherical KDE and its nearest-neighbor limit provide the anomaly-scoring principle.

\textbf{Private density summaries and transfer.}
DP-MERF releases bounded random-feature mean embeddings once and reuses their private output~\citep{harder2021dpmerf}. Wagner et al.\ develop locality-sensitive quantization for efficient private KDE, including mechanisms based on random Fourier features~\citep{wagner2023privatekde}. These are established privacy tools, not contributions claimed by \ouralg. We adapt their summary perspective to graph--text hyperspherical coordinates and target normal-reference calibration. Analytical Gaussian calibration~\citep{balle2018analytic} supplies the release noise. Our protected unit is one complete target reference graph, jointly including both views; the guarantee does not cover public-source training or query confidentiality. This differs from private model training and does not require federated optimization of a large text encoder.

\section{Problem Setup}
Let $G=(V,E,X)$ be a graph with node set $V$, edge set $E$, and optional node attributes $X$. In class-based GLAD benchmarks~\citep{liu2023good,wang2025unifying}, training uses normal graphs only, $\mathcal{D}_{\mathrm{tr}}^0=\{G_i\}_{i=1}^{N}$, while evaluation ranks held-out normal graphs against anomalous-class graphs. A detector outputs $s(G)$, where larger scores indicate stronger anomalousness. Single-domain scoring uses target-domain training normals as references. Zero-shot transfer uses source-domain normal references and no target training data. Few-shot calibration adds $k$ trusted target normal graphs to the reference set at inference time only; model parameters remain fixed.

\begin{figure*}[t]
    \centering
    \includegraphics[width=0.96\linewidth]{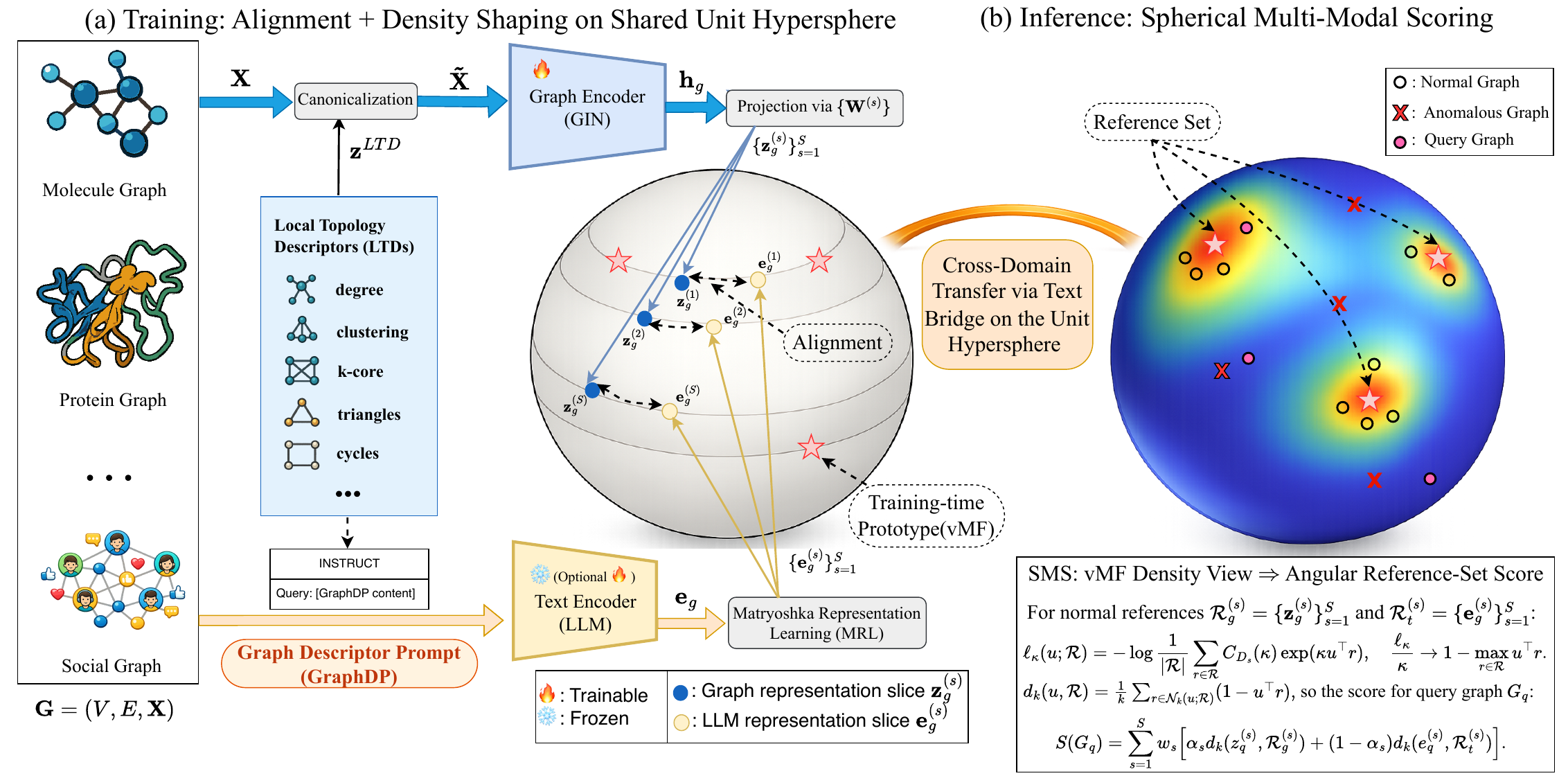}
    \caption{Overview of \ouralg. Local Topology Descriptors (LTDs) and graph attributes feed the graph encoder, while GraphDP embeds the same structural evidence through a frozen instruction-aware text encoder. Native Matryoshka text slices anchor graph-side alignment, and vMF prototypes shape training geometry. Inference uses angular normal-reference scoring; deployment regimes change reference construction. The optional text-adaptation branch denotes only the appendix LoRA diagnostic: text parameters are frozen in the main and private-reference experiments. The private extension replaces individual references with a released kernel summary (Section~\ref{sec:private_method}).}
    \label{fig:framework}
\end{figure*}

\section{Method}
\label{sec:method}

\ouralg follows a single principle: transferable graph-level anomaly detection should estimate normal density in a representation space whose coordinates remain comparable across graph families. To obtain such a space, as shown in Fig.~\ref{fig:framework},  \ouralg constructs two coupled views of each graph. The graph view maps topology and optional attributes into a structure-aware graph code. The language view serializes the same structural evidence into a compact Graph Descriptor Prompt (GraphDP) and embeds it with an instruction-aware text encoder. Training aligns the graph code to Matryoshka text anchors on the unit hypersphere and shapes the normal geometry with lightweight spherical prototypes. Inference then uses the same graph/text reference-set scorer in single-domain, zero-shot, and few-shot regimes; only the normal reference set changes.

\subsection{GraphDP and Dual Encoders}
\label{sec:graphdp_encoders}

\paragraph{Local topology descriptors and canonicalized node features.}
Let $G=(V,E,X)$ be an input graph, with optional node attributes $x_v$ for $v\in V$. For each node $v$, \ouralg first computes deterministic local topology descriptors (LTDs). Let $\mathcal{N}(v)$ denote the neighbor set and $\deg(v)$ the degree. The default LTD vector contains a degree profile and a local topological profile:
\begin{equation}
\begin{aligned}
z^{\mathrm{LTD}}_v= [
&d_v,\ d_v^{\min},\ d_v^{\mathrm{mean}},\ d_v^{\max},\ d_v^{\mathrm{std}},\\
&c(v),\ k_c(v),\ \widehat{T}(v),\ \widehat{C}_4(v)] .
\end{aligned}
\label{eq:ltd}
\end{equation}
where $d_v=\deg(v)$ and $(d_v^{\min},d_v^{\mathrm{mean}},d_v^{\max},d_v^{\mathrm{std}})$ are neighbor-degree statistics over $\mathcal{N}(v)$. Here $c(v)$ is the local clustering coefficient, $k_c(v)$ is the $k$-core index, and $\widehat{T}(v)$ and $\widehat{C}_4(v)$ denote triangle and 4-cycle estimates. These rooted-neighborhood statistics are available even when raw attributes are missing and provide a common structural vocabulary across molecules, proteins, and social graphs.

The LTD vector is concatenated with raw attributes when they exist. If attributes are absent, the input consists only of LTDs. The concatenated vector is normalized within each graph and passed through a small canonicalization MLP:
\begin{equation}
\tilde{x}_v
=
\phi_{\theta}\!\left(
\operatorname{znorm}_{G}
\bigl([z^{\mathrm{LTD}}_v;\,x_v]\bigr)
\right),
\label{eq:canon}
\end{equation}
where $\operatorname{znorm}_{G}$ denotes per-graph feature normalization over nodes. This step makes the graph encoder less sensitive to dataset-specific feature scales while preserving local structural evidence.

\paragraph{Structure-aware graph encoder.}
A GIN backbone~\citep{xu2019gin} maps the canonicalized node features into node representations. With $h_v^{(0)}=\tilde{x}_v$, the $\ell$-th message-passing layer is
\begin{equation}
\begin{aligned}
h_v^{(\ell)}=
\operatorname{MLP}_{\ell}\!\Big((1+\epsilon_{\ell})h_v^{(\ell-1)}+\sum_{u\in\mathcal{N}(v)}h_u^{(\ell-1)}
\Big).
\end{aligned}
\label{eq:gin}
\end{equation}
After $L$ layers, \ouralg concatenates mean and componentwise max readouts to obtain $r_g(G)\in\mathbb{R}^{2H}$:
\begin{equation}
\begin{aligned}
r_g(G)=\Big[&\frac{1}{|V|}\sum_{v\in V}h_v^{(L)};\ \max_{v\in V}h_v^{(L)}\Big].
\end{aligned}
\label{eq:readout}
\end{equation}
The graph-side global sketch $\xi(G) \in \mathbb{R}^{d_{\xi}}$ is intentionally lightweight and spectral rather than a duplicate of all GraphDP fields. Let
$
L_{\mathrm{sym}}=I-D^{-1/2}A_{\mathrm{sym}}D^{-1/2}
\label{eq:laplacian}
$
be the symmetric normalized Laplacian of the symmetrized graph. We define
\begin{equation}
\begin{aligned}
\xi(G)=\big[&\lambda_2,\ldots,\lambda_{r+1},\\
&Q_{0.1}(\{\rho_j\}_{j=1}^{q}),
Q_{0.5}(\{\rho_j\}_{j=1}^{q}),
Q_{0.9}(\{\rho_j\}_{j=1}^{q})\big].
\end{aligned}
\label{eq:spectral_sketch}
\end{equation}
where $\lambda_2,\ldots,\lambda_{r+1}$ are the first non-trivial small eigenvalues of $L_{\mathrm{sym}}$, $Q_p$ denotes the empirical $p$-quantile, and
$
\rho_j=\frac{u_j^{\top}L_{\mathrm{sym}}u_j}{u_j^{\top}u_j}
\label{eq:rayleigh}
$
are random Rayleigh quotients. With default $r=8$ and $q=16$, the sketch dimension is $d_{\xi}=r+3=11$. And the final graph evidence code is
\begin{equation}
h_g(G)=\big[r_g(G);p_{\xi}(\xi(G))\big]\in\mathbb{R}^{D_g},
\label{eq:graphcode}
\end{equation}
where $p_{\xi}$ is the spectral projection layer and $D_g=2H+d_{\xi}$.

\paragraph{GraphDP text anchors.}
The language view serializes graph-level aggregates of structural evidence into a compact key-value prompt. A GraphDP contains fields such as graph size, density, connected components, degree quantiles and entropy, clustering statistics, triangle and 4-cycle statistics, core-number summaries, and spectral summaries. It is deliberately short and structured; \ouralg does not rely on long chain-of-thought prompting or free-form reasoning.

Let $\operatorname{GraphDP}(G)$ denote this deterministic serialization. The default frozen text encoder is Qwen3-Embedding-0.6B~\citep{qwen3embedding} with $D_{\mathrm{text}}=1024$. Given a fixed instruction $I$ and the GraphDP query, the text encoder outputs
\begin{equation}
e(G)=f_{\mathrm{text}}\bigl(I,\operatorname{GraphDP}(G)\bigr)
\in\mathbb{R}^{D_{\mathrm{text}}}.
\label{eq:textencoder}
\end{equation}
The text encoder is used only as an embedding function, and its parameters are frozen by default.

Qwen3-Embedding supports native Matryoshka Representation Learning (MRL), so \ouralg obtains multi-resolution text anchors by prefix truncation. For slice dimensions $D_1<\cdots<D_S\le D_{\mathrm{text}}$, the $s$-th text slice is
\begin{equation}
e^{(s)}(G)=
\frac{e(G)_{1:D_s}}
{\|e(G)_{1:D_s}\|_2}
\in\mathbb{S}^{D_s-1}.
\label{eq:textslice}
\end{equation}
These frozen unit vectors serve as stable structural-language anchors for graph-side alignment.

\subsection{Graph-Language Alignment}
\label{sec:mrl_alignment}

For each slice $s$, a graph-side affine projection with $W^{(s)}\in\mathbb{R}^{D_s\times D_g}$ and $b^{(s)}\in\mathbb{R}^{D_s}$ maps the graph evidence code into the same spherical coordinate system as the corresponding text slice:
\begin{equation}
z^{(s)}(G)=
\frac{W^{(s)}h_g(G)+b^{(s)}}
{\|W^{(s)}h_g(G)+b^{(s)}\|_2}
\in\mathbb{S}^{D_s-1}.
\label{eq:graphslice}
\end{equation}
Only the graph side is trained. The text slices $e^{(s)}(G)$ remain fixed, which prevents the anchor space from drifting when the number of normal graphs is small.

For a minibatch $\mathcal{B}$ of normal graphs, \ouralg minimizes the multi-slice soft cosine alignment loss
\begin{equation}
\mathcal{L}_{\mathrm{align}}
=
\frac{1}{|\mathcal{B}|S}
\sum_{G\in\mathcal{B}}\sum_{s=1}^{S}
\left(
1-\left\langle z^{(s)}(G),e^{(s)}(G)\right\rangle
\right).
\label{eq:align}
\end{equation}
This objective aligns graph and language views at multiple resolutions: smaller slices emphasize coarse structural signals, while larger slices retain finer GraphDP evidence.

% Our default setting $D_s\in\{64,128,256,512\}$, so some projections are dimension-reducing and the largest slice is overcomplete. 
To avoid a rank-impossible orthogonality target, \ouralg uses a dimension-adaptive Gram regularizer to stabilize angular alignment:
\begin{equation}
\begin{aligned}
\mathcal{R}_{\mathrm{orth}}
=&\frac{1}{S}\sum_{s:D_s\le D_g}
\|W^{(s)}W^{(s)\top}-I_{D_s}\|_F^2\\
&+\frac{1}{S}\sum_{s:D_s>D_g}
\|W^{(s)\top}W^{(s)}-I_{D_g}\|_F^2 .
\end{aligned}
\label{eq:orth}
\end{equation}
For $D_s\le D_g$, this encourages orthonormal rows and stable dimensionality reduction. For $D_s>D_g$, it encourages an approximately isometric linear embedding. The penalty regulates $W^{(s)}$, not the bias, and does not guarantee angular preservation for arbitrary inputs. Unit normalization removes output magnitude; an affine projection with nonzero bias is not generally invariant to rescaling its input.

\subsection{Training-Time Spherical Density Shaping}
\label{sec:prototype_shaping}

In addition to alignment, \ouralg uses spherical prototypes as an optional training-time geometry regularizer. For each slice $s$, let
\[
\mathcal{M}^{(s)}=\{\mu^{(s)}_1,\ldots,\mu^{(s)}_K\},
\qquad
\|\mu^{(s)}_j\|_2=1,
\]
be $K$ normal prototype directions on the unit sphere. With learned positive concentration coefficients $\kappa_j^{(s)}$, the implemented hard-assignment energy is
\begin{equation}
E_{\mathrm{proto}}^{(s)}(G)
=
-\max_{1\le j\le K}
\kappa_j^{(s)}\left\langle z^{(s)}(G),\mu_j^{(s)}\right\rangle.
\label{eq:proto_energy}
\end{equation}
Let $E_{\mathrm{proto}}(G)=\sum_s E_{\mathrm{proto}}^{(s)}(G)$. After alignment warmup, prototype directions are initialized from training normals and updated by EMA of cosine-assigned normal embeddings. For each normal graph $G$ and its synthetic perturbation $G^{-}$, training regularizes normal energy and enforces a relative margin:
\begin{equation}
\begin{aligned}
\mathcal{L}_{\mathrm{proto}}
=&\frac{1}{|\mathcal{B}|}\sum_{G\in\mathcal{B}}
\Big\{\log\cosh\big(E_{\mathrm{proto}}(G)\big)\\
&+\big[m_{\mathrm{proto}}-
\big(E_{\mathrm{proto}}(G^{-})-E_{\mathrm{proto}}(G)\big)\big]_+\Big\}.
\end{aligned}
\label{eq:proto_loss}
\end{equation}
The perturbation indicator is used only inside the training loss and is never inserted into GraphDP. This vMF-inspired energy is a training surrogate, not a normalized density likelihood or the final anomaly scorer. The latter remains the nonparametric reference-set score.
% Prototypes are not the final anomaly detector; they only shape the representation geometry used by the nonparametric reference-set score.
So the graph-side training objective is
\begin{equation}
\mathcal{L}
=
\mathcal{L}_{\mathrm{align}}
+
\lambda_{\mathrm{orth}}\mathcal{R}_{\mathrm{orth}}
+
\lambda_{\mathrm{proto}}\mathcal{L}_{\mathrm{proto}} .
\label{eq:train_obj}
\end{equation}
All text embeddings in Eqs.~\eqref{eq:textencoder}--\eqref{eq:textslice} can be pre-computed or forwarded through the frozen encoder without gradient updates.

\subsection{Spherical Reference-Set Scoring}
\label{sec:reference_scoring}

After training, \ouralg freezes both encoders and scores a query graph by its angular density relative to normal references. Let $\mathcal{R}$ be the normal reference set for the deployment regime. For each slice $s$, define graph and text reference embeddings
\begin{equation}
\begin{aligned}
\mathcal{R}^{(s)}_g&=\{z^{(s)}(G_i):G_i\in\mathcal{R}\},\\
\mathcal{R}^{(s)}_t&=\{e^{(s)}(G_i):G_i\in\mathcal{R}\}.
\end{aligned}
\label{eq:refs}
\end{equation}
A vMF KDE score for any unit embedding $u\in\mathbb{S}^{D_s-1}$ and reference set $\mathcal{A}$ is
\begin{equation}
\ell_\kappa(u;\mathcal{A})
=
-\log
\frac{1}{|\mathcal{A}|}
\sum_{r\in\mathcal{A}}
C_{D_s}(\kappa)\exp\!\left(\kappa\langle u,r\rangle\right),
\label{eq:vmfkde}
\end{equation}
where $\kappa$ is the concentration parameter and $C_{D_s}(\kappa)$ is the vMF normalizing constant. The reported scorer uses the high-concentration reference-set limit of this density view. Let $\mathcal{N}_k(u;\mathcal{A})$ be the $k$ nearest references under cosine distance and define
\begin{equation}
d_k(u,\mathcal{A})
=
\frac{1}{k}
\sum_{r\in\mathcal{N}_k(u;\mathcal{A})}
\left(1-\langle u,r\rangle\right).
\label{eq:knn_distance}
\end{equation}
The final graph-level anomaly score for query $G_q$ is
\begin{equation}
\begin{aligned}
s(G_q)=
\sum_{s=1}^{S}w_s
\Big[
&\alpha_s\,
 d_k\bigl(z^{(s)}(G_q),\mathcal{R}^{(s)}_g\bigr)\\
&+
(1-\alpha_s)\,
 d_k\bigl(e^{(s)}(G_q),\mathcal{R}^{(s)}_t\bigr)
\Big].
\end{aligned}
\label{eq:fixedscore}
\end{equation}
Here $w_s$ are slice weights and $\alpha_s$ balances graph and text channels. In our default implementation, these weights and score calibration are estimated from normal references only; the appendix~\ref{app:reliability} gives the normal-reference reliability rule used in the reported scorer. Single-domain GLAD uses target-domain training normals as $\mathcal{R}$. Zero-shot transfer replaces $\mathcal{R}$ by source-domain normals and uses no target graphs for training. Few-shot calibration adds a small number of trusted target normals to $\mathcal{R}$, while the graph encoder, frozen text encoder, projections, and prototypes remain unchanged.

\begin{table*}[!t]
\centering
\refstepcounter{table}\label{tab:main_results}
\parbox{0.96\textwidth}{\centering\scriptsize
\textsc{Table~\thetable.} Single-domain GLAD results in AUROC (\%). The \ouralg row uses the spherical reference-set scorer. Average AUROC and rank are computed over all twelve datasets. Best, second, and third are shown in \best{bold}, \second{underline}, and a dagger $\dagger$ marker.}\\[4pt]
\begingroup
\renewcommand{\third}[1]{#1\textsuperscript{\ensuremath{\dagger}}}
\resizebox{\textwidth}{!}{%
\setlength{\tabcolsep}{1.2pt}
\begin{tabular}{@{}l|cccccccccccc|cc@{}}
\toprule
Method & MUTAG & PROTEINS & D\&D & ENZYMES & DHFR & BZR & COX2 & AIDS & IMDB-B & NCI1 & COLLAB & REDDIT-B & \shortstack{Avg.\\AUROC} & \shortstack{Avg.\\Rank} \\
\midrule
\multicolumn{15}{l}{\emph{Graph Kernel + Detector}} \\
PK-SVM    & \std{46.06}{0.47} & \std{49.43}{0.69} & \std{47.69}{0.24} & \std{52.45}{0.29} & \std{48.31}{0.47} & \std{46.67}{0.52} & \std{52.15}{0.16} & \std{50.93}{0.19} & \std{51.75}{0.30} & \std{51.39}{0.19} & \std{49.72}{0.60} & \std{48.36}{0.67} & 49.58 & 12.42 \\
PK-iF     & \std{47.98}{0.41} & \std{61.24}{0.34} & \std{75.29}{0.46} & \std{49.82}{0.67} & \std{52.79}{0.35} & \std{59.08}{0.29} & \std{52.48}{0.38} & \std{52.01}{0.53} & \std{52.83}{0.51} & \std{50.22}{0.12} & \std{51.38}{0.20} & \std{46.19}{0.21} & 54.28 & 11.25 \\
WL-SVM    & \std{62.18}{0.29} & \std{53.85}{0.26} & \std{47.98}{0.32} & \std{53.75}{0.34} & \std{50.30}{0.31} & \std{51.16}{0.36} & \std{53.34}{0.27} & \std{52.56}{0.41} & \std{52.98}{0.69} & \std{54.18}{0.67} & \std{54.62}{1.28} & \std{49.50}{0.54} & 53.03 & 10.67 \\
WL-iF     & \std{65.71}{0.38} & \std{65.75}{0.35} & \std{70.49}{0.28} & \std{51.03}{0.42} & \std{51.64}{0.22} & \std{51.71}{0.45} & \std{49.56}{0.11} & \std{61.42}{0.50} & \std{51.79}{0.32} & \std{50.41}{0.31} & \std{51.41}{0.39} & \std{49.84}{0.11} & 55.90 & 11.00 \\
\midrule
\multicolumn{15}{l}{\emph{GNN-based Deep Learning Methods}} \\
OCGIN     & \std{79.55}{0.22} & \std{76.46}{0.13} & \std{79.08}{0.19} & \std{62.44}{0.38} & \std{61.09}{0.27} & \std{69.13}{0.13} & \std{57.81}{0.50} & \std{96.89}{0.20} & \std{61.47}{0.18} & \std{69.46}{0.36} & \std{60.58}{0.27} & \std{82.10}{0.37} & 71.34 & 7.25 \\
GLocalKD  & \std{86.25}{0.57} & \third{\std{77.29}{0.41}} & \best{\std{80.76}{0.50}} & \std{61.75}{0.10} & \std{61.79}{0.54} & \std{68.55}{0.15} & \std{58.93}{0.47} & \std{96.93}{0.34} & \std{53.31}{0.53} & \std{65.29}{0.21} & \std{51.85}{0.18} & \std{80.32}{0.10} & 70.25 & 6.92 \\
OCGTL     & \std{88.02}{0.43} & \std{72.89}{0.57} & \std{77.76}{0.48} & \std{63.59}{0.11} & \std{59.82}{0.44} & \std{51.89}{0.46} & \std{59.81}{0.30} & \best{\std{99.36}{0.67}} & \std{65.27}{0.24} & \second{\std{75.75}{0.47}} & \std{48.13}{0.41} & \second{\std{88.03}{0.22}} & 70.86 & 6.33 \\
SIGNET    & \second{\std{88.84}{0.15}} & \std{75.86}{0.30} & \std{74.53}{0.11} & \std{63.12}{0.52} & \best{\std{72.87}{0.28}} & \second{\std{80.79}{0.38}} & \best{\std{72.35}{0.58}} & \std{97.60}{0.28} & \second{\std{70.12}{0.61}} & \std{74.32}{0.34} & \second{\std{72.45}{0.11}} & \std{85.24}{0.45} & \second{77.34} & 4.25 \\
GLADC     & \std{83.07}{0.29} & \second{\std{77.43}{0.19}} & \std{76.54}{0.25} & \std{63.44}{0.30} & \std{61.25}{0.19} & \std{68.23}{0.31} & \std{64.13}{0.23} & \std{98.02}{0.23} & \std{65.94}{0.26} & \std{68.32}{0.22} & \std{54.32}{0.37} & \std{78.87}{0.56} & 71.63 & 6.83 \\
CVTGAD    & \std{86.64}{0.32} & \std{76.49}{0.29} & \std{78.84}{0.40} & \second{\std{68.56}{0.43}} & \std{63.23}{0.38} & \std{77.69}{0.28} & \std{64.36}{0.16} & \third{\std{99.21}{0.27}} & \third{\std{69.82}{0.13}} & \std{69.13}{0.58} & \third{\std{71.01}{0.58}} & \third{\std{87.43}{0.60}} & 76.03 & 4.33 \\
MUSE      & \std{85.92}{0.28} & \std{76.87}{0.32} & \third{\std{79.23}{0.35}} & \third{\std{67.82}{0.38}} & \second{\std{71.45}{0.33}} & \std{78.34}{0.30} & \third{\std{65.87}{0.29}} & \std{98.95}{0.31} & \std{67.84}{0.27} & \third{\std{74.45}{0.26}} & \std{67.48}{0.36} & \std{85.32}{0.33} & \third{76.63} & \second{3.92} \\
UniFORM   & \third{\std{88.45}{0.24}} & \std{77.15}{0.27} & \std{78.56}{0.31} & \best{\std{69.34}{0.41}} & \third{\std{69.78}{0.36}} & \third{\std{79.85}{0.32}} & \std{65.12}{0.31} & \std{98.52}{0.22} & \std{68.42}{0.29} & \std{72.93}{0.31} & \std{66.23}{0.38} & \std{84.67}{0.35} & 76.58 & \third{4.08} \\
\midrule
\textbf{\ouralg} & \best{\std{91.42}{0.87}} & \best{\std{80.63}{0.42}} & \second{\std{79.41}{0.49}} & \std{66.20}{0.46} & \std{69.38}{0.37} & \best{\std{80.93}{0.38}} & \second{\std{68.70}{0.41}} & \second{\std{99.32}{0.59}} & \best{\std{77.20}{0.35}} & \best{\std{75.86}{0.56}} & \best{\std{80.29}{0.33}} & \best{\std{88.80}{0.69}} & \best{79.85} & \best{1.75} \\
\bottomrule
\end{tabular}}
\endgroup
\end{table*}

\section{Spherical Density Principle}
\begin{proposition}[vMF KDE to nearest-neighbor scoring]
Let $\mathcal{R}\subset \mathbb{S}^{d-1}$ be finite and let $u\in\mathbb{S}^{d-1}$. For the vMF KDE score in Eq.~\eqref{eq:vmfkde}, there exists a constant $a_\kappa$ independent of $u$ such that
\[
\lim_{\kappa\to\infty}
\left(\frac{\ell_\kappa(u;\mathcal{R})}{\kappa}+a_\kappa\right)
=1-\max_{r\in\mathcal{R}}\langle u,r\rangle .
\]
Thus, up to constants and positive scaling, high-concentration vMF scoring ranks queries by one-minus-cosine nearest-neighbor distance on the sphere.
\end{proposition}
\begin{proof}
Write $s_r=\langle u,r\rangle$ and $m(u)=\max_{r\in\mathcal{R}}s_r$. Since all reference terms are finite,
\[
\exp(\kappa m(u))\le
\sum_{r\in\mathcal{R}}\exp(\kappa s_r)
\le |\mathcal{R}|\exp(\kappa m(u)).
\]
Taking logarithms and dividing by $\kappa$ gives
\[
m(u)\le \frac{1}{\kappa}\log\sum_{r\in\mathcal{R}}\exp(\kappa s_r)
\le m(u)+\frac{\log|\mathcal{R}|}{\kappa}.
\]
Therefore the scaled log-sum-exp converges uniformly to $m(u)$. The remaining terms in Eq.~\eqref{eq:vmfkde}, namely $C_d(\kappa)$ and $|\mathcal{R}|$, are independent of $u$ and only contribute a constant $a_\kappa$ after scaling. Hence the negative log-density is asymptotically equivalent to $-m(u)$, or equivalently to $1-m(u)=\min_{r\in\mathcal{R}}(1-\langle u,r\rangle)$ because the additive constant $1$ does not change rankings.
\end{proof}
The proposition gives the exact 1-NN angular limit. The reported mean $k$-NN score with $k>1$ is a finite-sample robust extension of this limit: it reduces sensitivity to isolated reference points while preserving the same density intuition.

% BEGIN inlined privacy_method.tex
\section{Differentially Private Reference Calibration}
\label{sec:private_method}
The reference set is both the mechanism for target adaptation and a potential disclosure channel. Exact nearest-neighbor scoring retains individual normal embeddings, which a target institution may not be permitted to share. We introduce \textbf{GLASS-PR}, an optional private-reference extension: publicly available source encoders define the coordinates, while a target institution releases a noisy summary of its normal distribution. This changes the inference approximation, not the graph--language architecture.

\subsection{Protected Records and Public Coordinates}
Let $\mathcal R=\{G_i\}_{i=1}^n$ be a private target normal-reference set of fixed public size $n$. Two sets are adjacent if one complete graph record is replaced. A graph and its derived GraphDP, attributes, and all modality/slice embeddings form \emph{one} record. We assume a trusted local data holder, and that the source-trained graph encoder, frozen text encoder, feature map, and hyperparameters are fixed independently of $\mathcal R$. For every measurable output event $A$, the released mechanism must satisfy
\begin{equation}
\Pr[\mathcal M(\mathcal R)\in A]
\le e^\epsilon\Pr[\mathcal M(\mathcal R')\in A]+\delta.
\label{eq:dp_definition}
\end{equation}
This is target-reference graph-record privacy. It does not privatize source training, hide query graphs from a scoring service, or provide user-level privacy when one person contributes multiple graphs. In particular, ordinary target-domain GNN training cannot precede this mechanism while retaining the following sensitivity argument: changing a training record could change every encoded reference. Our private experiments therefore train fresh public-source models without target optimization.

\subsection{A Bounded Joint Spherical Summary}
For unit embeddings $u,r\in\mathbb S^{d-1}$, consider
\begin{equation}
K_\kappa(u,r)=\exp\{\kappa(\langle u,r\rangle-1)\}
=\exp\{-\kappa\|u-r\|_2^2/2\}.
\label{eq:dp_sphere_kernel}
\end{equation}
This bounded kernel differs from the vMF density kernel only by a positive, query-independent factor. Its Gaussian form permits random-feature mean summaries, following established private kernel methods~\citep{harder2021dpmerf,wagner2023privatekde}.

Index channels by $c=(s,m)$, where $m\in\{g,t\}$ and $u_{g,s}=z^{(s)}$, $u_{t,s}=e^{(s)}$. Publicly draw $M$ frequencies $\omega_{cj}\sim\mathcal N(0,\kappa I_{D_s})$ and define paired features
\begin{equation}
\phi_c(u)=\frac{1}{\sqrt M}
\bigoplus_{j=1}^M
\begin{bmatrix}\cos(\omega_{cj}^{\top}u)\\\sin(\omega_{cj}^{\top}u)\end{bmatrix}.
\label{eq:dp_rff}
\end{equation}
Then $\|\phi_c(u)\|_2=1$ exactly, and its expected inner product is $K_\kappa(u,r)$. Fix nonnegative public channel weights $a_{g,s}=w_s\alpha_s$, $a_{t,s}=w_s(1-\alpha_s)$ with $\sum_c a_c=1$. The joint feature and empirical mean are
\begin{equation}
\Phi(G)=\bigoplus_c\sqrt{a_c}\phi_c(u_c(G)),\qquad
\mu_{\mathcal R}=\frac1n\sum_{i=1}^n\Phi(G_i).
\label{eq:dp_joint}
\end{equation}
The weighted concatenation has norm one, so all modalities and slices are covered by one sensitivity bound. They cannot be treated as disjoint private datasets merely because their encoders differ. The private extension fixes $\alpha_s=1/2$ and uniform slice weights; the nonprivate reliability estimates in the appendix are \emph{not} reused without privacy accounting.

\begin{proposition}[Private reference summary]
\label{prop:private_summary}
Under the public-coordinate assumptions above, the replacement $\ell_2$ sensitivity of $\mu_{\mathcal R}$ is at most $2/n$. Consequently,
\begin{equation}
\widetilde\mu_{\mathcal R}=\mu_{\mathcal R}+\eta,\qquad
\eta\sim\mathcal N(0,\sigma_{\rm DP}^2 I),
\label{eq:dp_release}
\end{equation}
with analytically calibrated Gaussian noise is $(\epsilon,\delta)$-DP. Scores computed solely from this release and public information inherit that guarantee.
\end{proposition}
\begin{proof}
Replacing $G_i$ by $G_i'$ changes the mean by
$[\Phi(G_i)-\Phi(G_i')]/n$, whose norm is at most $2/n$ by the triangle inequality. Gaussian calibration~\citep{balle2018analytic} gives the stated release guarantee; postprocessing does not increase its privacy budget. This argument applies to the joint release, not separate full-budget releases for every channel.
\end{proof}

\subsection{Private Scoring and Source Shrinkage}
We rank queries using negative approximate kernel density,
\begin{equation}
s_{\rm PR}(G)=-\langle\Phi(G),\widetilde\mu_{\mathcal R}\rangle.
\label{eq:dp_score}
\end{equation}
Random features and Gaussian noise can produce negative density estimates, so we do not take their logarithm or claim that the signed estimate is a normalized probability density. Equation~\eqref{eq:dp_score} approximates a fixed mixture of bounded kernel similarities. It is distinct from exact mean $k$-NN and from averaging channel log-densities. We report these alternatives separately rather than selecting one using target test labels.

For small reference sets, a public-source mean supplies a shrinkage prior:
\begin{equation}
\widetilde\mu_{\rm mix}=(1-\lambda)\mu_{\rm src}
+\lambda\widetilde\mu_{\mathcal R},\qquad
\lambda=\frac{n}{n+n_0}.
\label{eq:dp_mix}
\end{equation}
Here $n_0$ is fixed publicly. This is postprocessing of the private summary, and source-only scoring ($\lambda=0$) is independent of the target reference set. Shrinkage reduces noise but retains source bias; it is not guaranteed to help under semantic shift. All score normalization, private-target bandwidth fitting, and private-target threshold estimation are excluded from this mechanism. Further private accesses would require additional accounting.

\paragraph{Resolution and distributed release.}
The sensitivity $2/n$ is independent of slice dimension. Although the summary noise norm grows with its dimension, a fixed unit-feature query has noise variance $\sigma_{\rm DP}^2$. Therefore smaller MRL slices do not automatically imply a stronger privacy guarantee. Separately, $J$ fixed disjoint institutions can each release a local mean and aggregate them with public size weights. For equal privacy parameters, independent local Gaussian noise gives aggregate standard deviation $\sqrt J$ times that of one trusted centralized release with the same total $n$. We simulate this noise cost, not an encrypted networking protocol or federated large-model training.

% END inlined privacy_method.tex

\section{Nonprivate Experiments}
\subsection{Setup}
We evaluate 12 widely used GLAD benchmarks~\citep{morris2020tudataset}: MUTAG, DHFR, BZR, COX2, AIDS, and NCI1 for small molecules; PROTEINS, D\&D, and ENZYMES for proteins; and IMDB-BINARY, COLLAB, and REDDIT-BINARY for social graphs. Training uses normal-class graphs only, and testing ranks held-out normal graphs against anomalous-class graphs. We report mean and standard deviation over five seeds, with AUROC as the primary metric. Appendix Table~\ref{tab:dataset_stats} summarizes graph counts, classes, average sizes, node-feature availability, and domain groups.

The baseline suite includes four graph-kernel detectors (PK-SVM, PK-iF, WL-SVM, WL-iF) and eight GNN-based methods (OCGIN, GLocalKD, OCGTL, SIGNET, GLADC, CVTGAD, MUSE, UniFORM). The main \ouralg row uses the spherical reference-set scorer in Eq.~\eqref{eq:fixedscore} across all datasets. Unless otherwise specified, the reported configuration uses a frozen Qwen3-Embedding-0.6B text encoder and native MRL slices $\{64,128,256,512\}$; larger frozen encoders and text adaptation are evaluated as ablations.

The text input is deterministic and label-free. A typical query is:
\begin{quote}\scriptsize\ttfamily\raggedright
Instruct: Encode the graph description for graph-level anomaly detection.\\
Query: domain=mol; phase=single; nodes=\{...\}; edges=\{...\}; density=\{...\}; degree\_q=\{...\}; clustering=\{...\}; motifs=triangles/4cycles; core=\{...\}; spectral=\{...\}.
\end{quote}
The domain and phase fields describe the deployment episode, while normal/anomaly status is never serialized.

\subsection{Single-Domain Results}
Table~\ref{tab:main_results} reports the full single-domain comparison against established graph-kernel and GNN baselines. \ouralg obtains the best average AUROC (79.85) and the best average rank (1.75) among the listed baselines. It ranks first on seven datasets (MUTAG, PROTEINS, BZR, IMDB-B, NCI1, COLLAB, and REDDIT-B) and remains top-two on D\&D, COX2, and AIDS. This result profile is consistent with the intended role of the graph-language bridge: graph neighborhoods preserve strong molecular and social structural signals, while frozen text anchors stabilize weakly attributed or structurally transferable datasets.

The compact 0.6B encoder provides a strong accuracy--capacity balance across the benchmark suite (Table~\ref{tab:scale}). Fine-grained biochemical tasks remain more sensitive to structural coverage; Appendix~\ref{app:structural_coverage} discusses these cases and the encoder-scaling trade-off.

\begin{table}[!t]
\centering
\caption{Text-encoder scaling (0.6B, 4B, 8B of Qwen3-Embedding Model) ablation in single-domain AUROC (\%) across all 12 datasets. The 0.6B column is the default setting of \ouralg.}
\label{tab:scale}
\resizebox{0.75\columnwidth}{!}{%
\begin{tabular}{@{}lccc@{}}
\toprule
Dataset & 0.6B & 4B & 8B \\
\midrule
MUTAG & \best{91.42$\pm$0.87} & 90.98$\pm$1.35 & 90.96$\pm$1.06 \\
PROTEINS & \best{80.63$\pm$0.42} & 80.10$\pm$0.76 & 79.46$\pm$0.81 \\
D\&D & 79.41$\pm$0.49 & \best{80.22$\pm$1.25} & 78.62$\pm$1.28 \\
ENZYMES & 66.20$\pm$0.46 & \best{70.51$\pm$0.25} & 68.42$\pm$0.27 \\
DHFR & \best{69.38$\pm$0.37} & 69.03$\pm$0.30 & 68.25$\pm$0.27 \\
BZR & \best{80.93$\pm$0.38} & 79.77$\pm$0.10 & 80.23$\pm$0.18 \\
COX2 & \best{68.70$\pm$0.41} & 67.30$\pm$0.31 & 68.18$\pm$0.38 \\
AIDS & 99.32$\pm$0.59 & 99.38$\pm$0.43 & \best{99.41$\pm$0.01} \\
IMDB-B & \best{77.20$\pm$0.35} & 76.84$\pm$0.38 & 75.15$\pm$0.39 \\
NCI1 & 75.86$\pm$0.56 & 75.63$\pm$0.27 & \best{76.49$\pm$0.27} \\
COLLAB & 80.29$\pm$0.33 & 79.16$\pm$0.33 & \best{80.42$\pm$0.28} \\
REDDIT-B & \best{88.80$\pm$0.69} & 88.09$\pm$0.74 & 87.38$\pm$0.36 \\
\bottomrule
\end{tabular}}
\end{table}

\subsection{Cross-Domain and Few-Shot Transfer}
Table~\ref{tab:cross} evaluates directional transfer by changing the normal reference set while keeping the scoring rule fixed. The strongest regime is molecular-to-protein transfer: source molecular references improve PROTEINS by +2.9pp and D\&D by +4.6pp over their matched single-domain references, indicating that GraphDP fields such as motifs, degree profiles, density, and low-frequency connectivity can carry compatible substructure semantics across these domains. Protein-to-molecule and social-to-molecule transfer are also stable on AIDS and MUTAG, where the target normal geometry is well covered by the source references. In contrast, transfer into REDDIT-BINARY loses 15--20pp, confirming that social interaction graphs can assign different anomaly meaning to the same structural statistics. All-domain references improve PROTEINS slightly but dilute MUTAG and IMDB-B, suggesting that source coverage and source weighting matter in broad heterogeneous mixtures.

\begin{table}[t]
\centering
\caption{Zero-shot cross-domain AUROC (\%). Target datasets marked $\dagger$ are scored without target-domain training.}
\label{tab:cross}
\resizebox{\columnwidth}{!}{%
\begin{tabular}{@{}llccc@{}}
\toprule
Source $\to$ Target & Dataset & Transfer & Single & $\Delta$ \\
\midrule
\multirow{2}{*}{Mol $\to$ Social}
 & IMDB-B$^\dagger$   & \std{74.65}{0.38} & \std{77.20}{0.35} & {$-$2.6} \\
 & REDDIT-B$^\dagger$ & \std{69.34}{1.61} & \std{88.80}{0.69} & {$-$19.5} \\
\midrule
\multirow{2}{*}{Social $\to$ Mol}
 & MUTAG$^\dagger$ & \std{90.95}{0.77} & \std{91.42}{0.87} & {$-$0.5} \\
 & AIDS$^\dagger$  & \best{\std{99.38}{0.30}} & \std{99.32}{0.59} & \best{{+0.06}} \\
\midrule
\multirow{3}{*}{Mol $\to$ Protein}
 & PROTEINS$^\dagger$ & \best{\std{83.55}{0.30}} & \std{80.63}{0.42} & \best{{+2.9}} \\
 & D\&D$^\dagger$     & \best{\std{83.99}{1.08}} & \std{79.41}{0.49} & \best{{+4.6}} \\
 & ENZYMES$^\dagger$  & \std{62.84}{0.21} & \std{66.20}{0.46} & {$-$3.4} \\
\midrule
\multirow{2}{*}{Protein $\to$ Mol}
 & MUTAG$^\dagger$ & \best{\std{91.81}{1.13}} & \std{91.42}{0.87} & \best{+0.4} \\
 & AIDS$^\dagger$  & \best{\std{99.38}{0.30}} & \std{99.32}{0.59} & \best{{+0.06}} \\
\midrule
\multirow{2}{*}{Protein $\to$ Social}
 & IMDB-B$^\dagger$   & \std{73.62}{0.48} & \std{77.20}{0.35} & {$-$3.6} \\
 & REDDIT-B$^\dagger$ & \std{73.43}{1.67} & \std{88.80}{0.69} & {$-$15.4} \\
\midrule
\multirow{3}{*}{All-domain}
 & PROTEINS & \best{\std{81.35}{0.29}} & \std{80.63}{0.42} & \best{+0.7} \\
 & IMDB-B   & \std{73.61}{0.49} & \std{77.20}{0.35} & {$-$3.6} \\
 & MUTAG    & \std{89.06}{0.54} & \std{91.42}{0.87} & {$-$2.4} \\
\bottomrule
\end{tabular}}
\end{table}

Few-shot calibration uses target normal graphs only as additional references. Table~\ref{tab:fewshot} shows that PROTEINS is the clearest reference-calibration case: one trusted target normal graph improves AUROC from 65.59 to 73.33, and 16--32 target normals bring the score within about one point of the single-domain reference without gradient updates. MUTAG improves more gradually after the first few references, reaching 84.54 at 32 shots, which indicates that a single normal graph may be unrepresentative for small molecular datasets. IMDB-BINARY remains difficult in the few-shot regime; its full target-reference row improves substantially, but small reference sets do not cover the broader social-graph manifold.

\begin{table}[t]
\centering
\caption{Few-shot target calibration (AUROC \%). Parameters are frozen; $k$ target normal graphs are used only as references.}
\label{tab:fewshot}
\resizebox{0.8\columnwidth}{!}{%
\begin{tabular}{@{}lccc@{}}
\toprule
$k$ & MUTAG$^\dagger$ & PROTEINS$^\dagger$ & IMDB-B$^\dagger$ \\
\midrule
0  & 74.35$\pm$3.68 & 65.59$\pm$3.56 & 63.47$\pm$3.68 \\
1  & 68.15$\pm$3.03 & 73.33$\pm$3.04 & 58.97$\pm$2.93 \\
4  & 68.96$\pm$1.88 & 74.95$\pm$2.01 & 62.09$\pm$2.47 \\
8  & 77.12$\pm$1.47 & 78.24$\pm$1.37 & 62.82$\pm$1.60 \\
16 & 80.05$\pm$1.39 & 79.69$\pm$1.11 & 63.50$\pm$1.61 \\
32 & 84.54$\pm$1.18 & 79.81$\pm$0.93 & 63.99$\pm$1.34 \\
Full & \second{88.90$\pm$1.13} & \best{81.95$\pm$0.73} & \second{73.94$\pm$0.56} \\
Single-domain & \best{91.42$\pm$0.87} & \second{80.63$\pm$0.42} & \best{77.20$\pm$0.35} \\
\bottomrule
\end{tabular}}
\end{table}

\subsection{Ablation Analysis}
The ablations clarify where the gains come from. Appendix Table~\ref{tab:scoring} shows that no single channel dominates all datasets. Text-only neighborhoods are already strong on AIDS, graph neighborhoods dominate MUTAG and DHFR, and the full scorer is strongest on PROTEINS, MUTAG, and IMDB-B among the representative targets. Finite-\(\kappa\) vMF and SMS variants provide the density interpretation, but their performance is not uniformly better than angular reference-set scoring; this supports using the high-concentration k-NN limit as the reported detector. The text-encoder scaling ablation in Table~\ref{tab:scale} shows a specificity-transferability trade-off rather than monotone improvement with model size. Finally, the LoRA ablation in Appendix Table~\ref{tab:lora} shows that adapting the text encoder is highly dataset-dependent, with large degradations on PROTEINS, AIDS, IMDB-BINARY, and COLLAB; this supports freezing the instruction-aware text space and learning graph-side alignment only.

\subsection{Discussion: What Transfers?}
The results identify four empirical regimes. First, related biochemical domains transfer well when the source references cover target substructure semantics; this explains the Mol\(\to\)Protein gains on PROTEINS and D\&D. Second, some molecular targets are already well covered by protein or social references, leading to near-lossless transfer on MUTAG and AIDS. Third, few-shot calibration is effective when the target normal manifold is related but locally miscalibrated, as seen on PROTEINS. Fourth, social targets, especially REDDIT-BINARY, remain hard under biological sources because community and interaction semantics differ from biochemical substructures. These regimes are summarized in Appendix Table~\ref{tab:regimes}.

This interpretation also clarifies the role of training-time geometry. Prototypes encourage compact normal modes during training, but the final detector remains a nonparametric reference-set score. Frozen text anchors keep the structural-language coordinate system stable when the reference set changes.

\FloatBarrier
\FloatBarrier
% BEGIN inlined privacy_experiments.tex
\section{Private-Reference Experiments}
\label{sec:private_experiments}
\subsection{Independent Protocol and Scope}
We conduct a separate, fully rerun experiment series to assess the private-reference mechanism, rather than applying noise to a previously selected test score. Two public source pools contain molecular normals (AIDS, MUTAG, NCI1) or protein normals (PROTEINS, D\&D, ENZYMES). For each pool, four models are trained from scratch with seeds 42--45 for exactly 150 epochs. Target graphs and target metrics are not accessed during training or checkpoint selection. Each finalized model is evaluated on the nine datasets outside its source pool, giving 18 source--target dataset settings and 72 model--target episodes covering all twelve benchmarks. These source models, their fixed input padding, and their selection policy are distinct from the preceding nonprivate series; comparisons in this section use their own matched controls.

The graph backbone is a three-layer GIN with mean/max pooling, the 11-dimensional spectral sketch, and native text slices $\{64,128,256,512\}$. Text anchors are content-checked cached embeddings of the frozen 0.6B encoder. No text fine-tuning is performed. The main private summary uses slices $\{64,128\}$, $M=128$ frequencies per channel, uniform joint weights, and $\kappa=4$. Thus $P=1024$ float32 coordinates require 4\,KiB per released summary, excluding the shared encoders. All parameters are fixed across targets. We evaluate $\epsilon\in\{1,2,4,8\}$ with $\delta=10^{-5}$ and ten noise draws per condition; nonprivate summaries are separate controls, not infinite-budget candidates for selection.

Target references are trusted normals from the 80\% normal training pool; test normals and anomalous-class graphs remain disjoint. We evaluate $n\in\{1,4,8,16,32,64,128,256\}$ when available. The primary comparison uses $n=128$, except MUTAG ($n=100$). Ordered-edge/attribute hashes remove exact matches to source training graphs before target evaluation; this check does not establish full graph-isomorphism disjointness. The published TU data are used as privacy simulations, not as clinical patient records. All target labels are used only to define benchmark splits and compute final metrics, never to select encoders, kernels, modalities, or noise draws.

% BEGIN inlined private_results/main_table.tex
\begin{table*}[t]
\centering
\footnotesize
\setlength{\tabcolsep}{3pt}
\renewcommand{\arraystretch}{1.08}
\caption{Private target-reference calibration, AUROC (\%). All columns use independently retrained public-source encoders. NN denotes angular mean 5-NN; Exact KDE and Sketch use $\kappa=4$. Private columns use the same target-only sketch, $\delta=10^{-5}$, and one total graph--text budget. Values are mean $\pm$ sample standard deviation over four model/split seeds, after averaging ten noise draws per private condition. No checkpoint or scorer is selected using target labels.}
\label{tab:dp_main}
\resizebox{\textwidth}{!}{%
\begin{tabular}{@{}llrccccccc@{}}
\toprule
Source & Target & $n$ & Source NN & Target NN & Exact KDE & Sketch & $\epsilon=1$ & $\epsilon=4$ & $\epsilon=8$ \\
\midrule
Mol & PROTEINS & 128 & $60.55\!\pm\!5.18$ & $72.34\!\pm\!1.85$ & $74.20\!\pm\!1.83$ & $74.03\!\pm\!1.93$ & $70.86\!\pm\!2.34$ & $73.78\!\pm\!1.87$ & $73.89\!\pm\!2.00$ \\
Mol & D\&D & 128 & $59.54\!\pm\!6.19$ & $71.74\!\pm\!2.48$ & $66.67\!\pm\!3.46$ & $66.86\!\pm\!3.27$ & $64.04\!\pm\!2.45$ & $66.82\!\pm\!3.34$ & $66.78\!\pm\!3.43$ \\
Protein & AIDS & 128 & $29.69\!\pm\!14.76$ & $98.74\!\pm\!0.96$ & $89.75\!\pm\!5.83$ & $88.41\!\pm\!5.65$ & $82.31\!\pm\!4.50$ & $88.67\!\pm\!4.60$ & $88.32\!\pm\!5.54$ \\
Mol & IMDB-B & 128 & $47.55\!\pm\!5.76$ & $57.85\!\pm\!2.87$ & $59.91\!\pm\!4.88$ & $59.93\!\pm\!4.89$ & $57.95\!\pm\!3.21$ & $59.70\!\pm\!4.54$ & $59.89\!\pm\!4.91$ \\
\bottomrule
\end{tabular}}
\end{table*}

% END inlined private_results/main_table.tex

\subsection{Private Calibration Utility}
Table~\ref{tab:dp_main} evaluates private calibration against matched reference scorers. At $\epsilon=4$, private target-only AUROC is 73.78 on Mol$\to$PROTEINS, 66.82 on Mol$\to$D\&D, 88.67 on Protein$\to$AIDS, and 59.70 on Mol$\to$IMDB-B. Their matched nonprivate summaries obtain 74.03, 66.86, 88.41, and 59.93, respectively. These results show that a compact, once-released summary can retain useful target-normal information without exposing individual reference embeddings or updating either encoder.

Across all 18 settings, adding $\epsilon=4$ noise changes the macro-average AUROC from 60.23 for the identical nonprivate summary to 60.08, a 0.15-point difference. This measures the added privacy cost, not equivalence to the original exact 5-NN detector: finite-bandwidth and summary approximation introduce a separate cost. The complete comparison and difficult transfer settings are analyzed in Appendix~\ref{app:private_failures} and Table~\ref{tab:dp_all}.

\begin{figure*}[t]
\centering
\includegraphics[width=\textwidth]{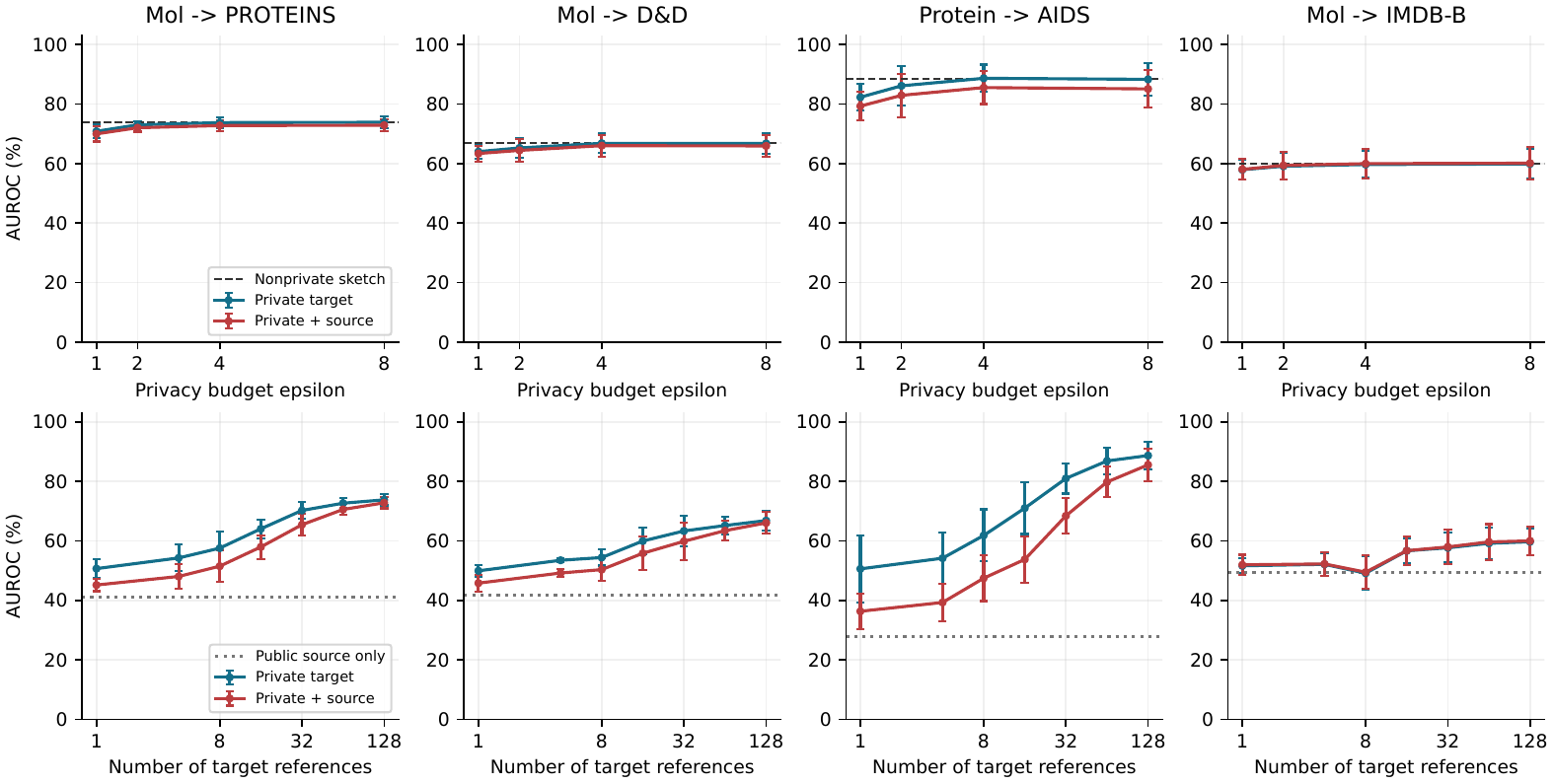}
\caption{Private reference calibration with independently trained public-source models. Top: privacy-budget sweep at $n=128$; the dashed line is the identical nonprivate target summary. Bottom: reference-size sweep at $\epsilon=4$; the dotted line uses only the public source summary. Source shrinkage fixes $\lambda=n/(n+32)$. Error bars are sample standard deviations across four model/split seeds after averaging ten noise draws, not forty independent training runs.}
\label{fig:private_tradeoffs}
\end{figure*}

\subsection{Reference Coverage, Modalities, and Public Priors}
Figure~\ref{fig:private_tradeoffs} illustrates how privacy budget and reference coverage support calibration. At $n=128$, increasing $\epsilon$ from 1 to 4 raises Mol$\to$PROTEINS AUROC from 70.86 to 73.78 and Protein$\to$AIDS from 82.31 to 88.67. Increasing the number of trusted references also improves calibration without gradient updates: it both expands normal coverage and reduces the $2/n$ sensitivity of the summary. The small-reference regime and its variability are discussed in Appendix~\ref{app:private_failures}.

% BEGIN inlined private_results/modalities.tex
\begin{table}[t]
\centering\scriptsize
\setlength{\tabcolsep}{3pt}
\caption{Private modality and source-prior diagnostics at $n=128$, $\epsilon=4$. AUROC averaged over four seeds and ten draws. All methods use the same total privacy budget.}
\label{tab:dp_modalities}
\begin{tabular}{@{}lrrrr@{}}
\toprule
Target & Graph & Text & Joint & Joint+source \\
\midrule
PROTEINS & 71.57 & 53.71 & 73.78 & 72.79 \\
D\&D & 65.79 & 54.54 & 66.82 & 66.03 \\
AIDS & 90.92 & 79.51 & 88.67 & 85.55 \\
IMDB-B & 59.05 & 57.17 & 59.70 & 59.99 \\
\bottomrule
\end{tabular}
\end{table}

% END inlined private_results/modalities.tex

The joint representation remains useful under a single privacy budget (Table~\ref{tab:dp_modalities}). On Mol$\to$PROTEINS, the joint summary obtains 73.78 versus 71.57 for graph-only and 53.71 for text-only, and it also improves over either single modality on Mol$\to$D\&D. All channels of a graph share the same budget through joint normalization. We retain one fixed mixture across targets; dataset-specific modality behavior is documented in Appendix~\ref{app:private_failures} rather than used for test-label-based selection.

Public-source mixing provides a complementary, privacy-preserving calibration option because it is postprocessing of the released target summary. With the fixed shrinkage rule, Mol$\to$IMDB-B reaches 59.99 at $n=128$, compared with 59.70 for target-only calibration. Its utility depends on source compatibility; Figure~\ref{fig:private_tradeoffs} reports both variants, and Appendix~\ref{app:private_failures} examines cases where retaining the source prior is unhelpful.

\subsection{Resolution, Distribution, and Cost}
Additional diagnostics vary slice resolution, feature count, and concentration while keeping the main configuration fixed (Appendix Table~\ref{tab:dp_resolution}). They test whether the useful geometry survives compression, not whether lower input dimension automatically strengthens differential privacy. Appendix Table~\ref{tab:dp_extra} reports AUPRC, FPR95, and the utility cost of independent two- and four-institution releases at fixed total reference size. The latter is an exact Gaussian aggregate-noise simulation under fixed disjoint site membership; it does not implement secure aggregation, network communication, or private federated model training.

The full series contains 210,852 scoring records, including matched nonprivate controls and all diagnostic settings. Existing structural and frozen-text caches make repeated evaluation inexpensive: summed training-loop device time is 1,205.6 seconds across eight source models and summed scoring-loop device time is 626.5 seconds across eight workers. Their sum is 0.51 GPU-hours on RTX 5090 hardware, excluding historical cache construction, model/data loading, and embedding export. This is measured worker time, not a claim about cold-start end-to-end runtime. The computation confirms that privacy--utility evaluation can be extensive without retraining the text model.

% END inlined privacy_experiments.tex

\FloatBarrier
\section{Conclusion}
We presented \ouralg, a framework for transferable graph-level anomaly detection through graph--language alignment and spherical reference-set scoring. GraphDP connects graph structure to a shared textual vocabulary, while frozen instruction-aware anchors and native Matryoshka slices provide coordinates for multimodal alignment. The resulting vMF-inspired density view unifies single-domain detection, zero-shot transfer, and few-shot reference calibration without changing the core architecture. Experiments across twelve benchmarks demonstrate strong single-domain performance and the utility of reference-based adaptation without target-domain gradient updates. Building on this representation, the private-reference extension supports graph-record differential privacy through a bounded joint graph--text kernel summary, with encoders fixed independently of private target references. Its evaluation quantifies privacy--utility trade-offs while distinguishing privacy noise from density approximation. Future work will enrich fine-grained structural vocabularies and improve source weighting and private density summaries under larger domain shifts.

\section*{Acknowledgments}
This work was supported in part by the National Natural Science Foundation of China under Grant No.62376236, the General Program of the Natural Science Foundation of Guangdong Province under Grant No.2024A1515011771, Shenzhen Science and Technology Program ZDSYS20230626091302006 (Shenzhen Key Lab of Multi-Modal Cognitive Computing), and the Shenzhen Stability Science Program 2023.

\FloatBarrier
% Bibliography inlined from the verified IEEEtran-generated main.bbl.

\newpage
\appendices
\section{Dataset Details for \ouralg Benchmark}
In Table~\ref{tab:dataset_stats}, we give detailed descriptions of the 12 graph datasets used in our study. The datasets span \textit{small-molecule graphs}, \textit{protein graphs}, and \textit{social networks}. For each dataset we describe the domain background, the classes used for anomaly detection, and the available node features. We report node features assuming \texttt{use\_node\_attr=True} when loading from TUDataset, which includes both discrete node labels and continuous attributes when available.

\begin{table*}[t]
\centering
\caption{Statistics of the datasets  \citep{morris2020tudataset} used in our experiments. Each graph-level dataset is characterized by the number of graphs, classes, average number of nodes and edges, and node feature dimensions. The ``Node Labels'' column indicates discrete node types (e.g., atom or amino-acid types), ``Node Attributes'' indicates additional continuous attributes, and ``Node Features'' shows the final dimensionality when loading with \texttt{use\_node\_attr=True}. Values in parentheses show dimensions without attribute loading.}
\label{tab:dataset_stats}
\resizebox{0.8\textwidth}{!}{%
\begin{tabular}{@{}lcccccccc@{}}
\toprule
\textbf{Dataset} & \textbf{\# Graphs} & \textbf{\# Classes} & \textbf{Avg.\ Nodes} & \textbf{Avg.\ Edges} & \textbf{Node Labels} & \textbf{Node Attributes} & \textbf{Node Features} & \textbf{Domain} \\ 
 & & & & & & & \textbf{(w/o attributes)} & \\
\midrule
\textsc{PROTEINS}   & 1\,113 & 2 & 39.1  & 72.8     & 3 SSE types      & 1         & 4 (3)      & Protein \\
\textsc{D\&D}       & 1\,178 & 2 & 284.3 & 715.7    & 89 AA types      & --        & 89 (89)    & Protein \\
\textsc{ENZYMES}    & 600    & 6 & 32.6  & 62.1     & 3 SSE types      & 18        & 21 (3)     & Protein \\
\textsc{MUTAG}      & 188    & 2 & 17.9  & 19.8     & 7 atom types     & --        & 7 (7)      & Molecule \\
\textsc{DHFR}       & 756    & 2 & 42.4  & 44.5     & Atom types       & 3         & 56 (53)    & Molecule \\
\textsc{BZR}        & 405    & 2 & 35.8  & 38.4     & Atom types       & 3         & 56 (53)    & Molecule \\
\textsc{COX2}       & 467    & 2 & 41.2  & 43.5     & Atom types       & 3         & 38 (35)    & Molecule \\
\textsc{AIDS}       & 2\,000 & 2 & 25.7  & 27.9     & Atom types       & 4         & 42 (38)    & Molecule \\
\textsc{NCI1}       & 4\,110 & 2 & 29.8  & 32.3     & 37 atom types    & --        & 37 (37)    & Molecule \\
\textsc{IMDB-B}     & 1\,000 & 2 & 19.8  & 96.5     & --               & --        & 0 (0)      & Social network \\
\textsc{COLLAB}     & 5\,000 & 3 & 74.5  & 2\,457.8 & --               & --        & 0 (0)      & Social network \\
\textsc{REDDIT-B}   & 2\,000 & 2 & 429.6 & 497.8    & --               & --        & 0 (0)      & Social network \\
\bottomrule
\end{tabular}%
}
\end{table*}

\paragraph{Node Feature Implementation Notes}
\begin{itemize}
    \item The \texttt{use\_node\_attr=True} flag in PyTorch Geometric's TUDataset loader determines whether continuous node attributes are concatenated with discrete node labels.
    \item For small-molecule datasets (DHFR, BZR, COX2, AIDS), the additional attributes typically represent 3D coordinates or physicochemical properties.
    \item For ENZYMES, the 18 additional attributes capture various amino acid properties including length, hydrophobicity, and secondary structure characteristics.
    \item PROTEINS includes a single continuous attribute representing physicochemical properties of secondary structure elements.
    \item Social network datasets lack intrinsic node features; practitioners often compute structural features externally.
\end{itemize}

\section{Protocol and Ablation Details}
\subsection{Evaluation protocol.}
The evaluation protocol separates representation learning, reference construction, and score evaluation. Single-domain experiments train on target-domain normal graphs; zero-shot transfer replaces the reference set by source-domain normals; few-shot calibration appends trusted target normals at inference. Test anomaly labels are not inputs to GraphDP or reference-weight estimation. The independently rerun private-reference experiments additionally enforce fixed-final-epoch checkpoint selection without any target metric access.

\begin{center}
\resizebox{\columnwidth}{!}{%
\begin{tabular}{@{}ll@{}}
\toprule
Item & Setting in reported experiments \\
\midrule
Main score & GLASS angular $k$-NN reference scorer \\
Ablation scores & graph-only, text-only, fusion, vMF, and SMS variants \\
Target anomaly labels & never serialized in GraphDP prompts \\
Domain/phase information & allowed as episode-level context \\
Synthetic-negative status & used by training loss; not serialized as text label \\
Zero-shot target data & evaluated only; not inserted into references \\
Few-shot target normals & inserted as references; no gradient updates \\
\bottomrule
\end{tabular}}
\end{center}

\subsection{Structural coverage and encoder scaling}
\phantomsection\label{app:structural_coverage}
ENZYMES and DHFR involve fine-grained biochemical distinctions that compact GraphDP fields may not fully represent. The measured performance gaps are consistent with this explanation, although the benchmarks alone do not establish a causal mechanism. Table~\ref{tab:scale} shows that increasing text-encoder size improves ENZYMES but does not uniformly close the gaps. Richer substructure fields and larger encoders should therefore be assessed jointly with transfer coverage, rather than selected from target test results. All twelve datasets remain in the main comparison; this discussion complements, rather than filters, that evidence.

\subsection{Normal-reference reliability weights}
\phantomsection\label{app:reliability}
The main scorer in Eq.~\eqref{eq:fixedscore} uses weights computed from normal references before evaluation. A fixed ablation sets $w_s=\log D_s/\sum_{u=1}^{S}\log D_u$ and $\alpha_s=1/2$. The default rule estimates reliability from normal-reference compactness and cross-view consistency. Let $u_{g,s}(G)=z^{(s)}(G)$, $u_{t,s}(G)=e^{(s)}(G)$, and let $\bar m$ denote the other modality for $m\in\{g,t\}$. For each normal reference $G_i\in\mathcal{R}$,
\begin{equation}
a_{m,s,i}=d_k\!\left(u_{m,s}(G_i),
\mathcal{R}_{m}^{(s)}\setminus\{u_{m,s}(G_i)\}\right)
\label{eq:loo_score}
\end{equation}
measures leave-one-out compactness. Let $\mu_{m,s}$, $\sigma_{m,s}$, and $\operatorname{IQR}_{m,s}$ be the mean, standard deviation, and interquartile range of these scores. Cross-view neighborhood consistency is
\begin{equation}
b_{m,s,i}=1-\frac{1}{k}
\sum_{j\in\operatorname{NN}_k(u_{m,s}(G_i))}
\left\langle
u_{\bar m,s}(G_i),u_{\bar m,s}(G_j)
\right\rangle,
\label{eq:cross_view_agree}
\end{equation}
where neighbors are selected among normal references in modality $m$. The graph-text alignment residual is
\begin{equation}
q_{s,i}=1-\left\langle z^{(s)}(G_i),e^{(s)}(G_i)\right\rangle .
\label{eq:align_resid}
\end{equation}
The reliability penalty and channel reliability are
\begin{equation}
\begin{aligned}
\pi_{m,s}=&\ \mu_{m,s}
+\frac{1}{2}\left(\sigma_{m,s}+\frac{1}{2}\operatorname{IQR}_{m,s}\right)\\
&+\left(\bar b_{m,s}+\frac{1}{2}\sigma(b_{m,s})\right)
+\frac{1}{2}\left(\bar q_s+\frac{1}{4}\sigma(q_s)\right),\\
\rho_{m,s}=&\operatorname{clip}_{10,90}
\left(\frac{\log D_s}{\max(\pi_{m,s},\varepsilon)}\right).
\end{aligned}
\label{eq:rel_score}
\end{equation}
Here $\operatorname{clip}_{10,90}$ denotes percentile clipping across slice/modality reliability values. The final slice and modality weights are
\begin{equation}
w_s=
\frac{(\rho_{g,s}+\rho_{t,s})\log D_s}
{\sum_{u=1}^{S}(\rho_{g,u}+\rho_{t,u})\log D_u},
\qquad
\alpha_s=
\frac{\rho_{g,s}}{\rho_{g,s}+\rho_{t,s}} .
\label{eq:reliability_weights}
\end{equation}
All statistics in Eqs.~\eqref{eq:loo_score}--\eqref{eq:reliability_weights} use normal references only. Appendix Table~\ref{tab:scoring} shows that this rule adapts to different modality regimes without using anomaly labels.

\begin{table}[t]
\centering
\caption{Scoring ablations (AUROC \%, mean$\pm$std over five seeds). Graph/Text/Fusion use individual or fused angular $k$-NN channels; vMF and SMS instantiate the density view; Fixed uses equal graph/text weights with normal-reference calibration; \ouralg uses the normal-reference reliability weights here.}
\label{tab:scoring}
\resizebox{\columnwidth}{!}{%
\begin{tabular}{@{}lccccccc@{}}
\toprule
Dataset & Graph & Text & Fusion & vMF & SMS & Fixed & \ouralg \\
\midrule
AIDS      & 95.52$\pm$3.29 & 99.79$\pm$0.01 & 97.09$\pm$2.02 & 90.67$\pm$5.86 & 93.92$\pm$3.32 & 99.77$\pm$0.03 & 99.32$\pm$0.59 \\
IMDB-B    & 68.64$\pm$0.39 & 74.12$\pm$0.51 & 74.13$\pm$0.40 & 68.88$\pm$2.39 & 65.74$\pm$1.11 & 77.29$\pm$0.35 & 77.20$\pm$0.35 \\
PROTEINS  & 74.83$\pm$0.91 & 71.31$\pm$1.25 & 77.08$\pm$0.98 & 69.19$\pm$1.65 & 79.70$\pm$0.58 & 80.63$\pm$0.42 & 80.63$\pm$0.42 \\
MUTAG     & 91.07$\pm$1.10 & 74.20$\pm$1.90 & 83.19$\pm$1.42 & 68.96$\pm$0.71 & 82.30$\pm$2.13 & 88.77$\pm$0.74 & 91.42$\pm$0.87 \\
ENZYMES   & 69.12$\pm$0.40 & 62.65$\pm$0.26 & 69.91$\pm$0.64 & 64.41$\pm$0.89 & 62.58$\pm$1.09 & 66.20$\pm$0.46 & 66.20$\pm$0.46 \\
DHFR      & 70.18$\pm$0.39 & 59.09$\pm$0.25 & 68.25$\pm$0.44 & 61.33$\pm$0.66 & 58.26$\pm$0.58 & 67.80$\pm$0.38 & 69.38$\pm$0.37 \\
\bottomrule
\end{tabular}}
\end{table}

\textbf{Algorithmic summary.}
The following pseudocode highlights that GLASS uses one representation pipeline and one scoring family; deployment changes only the reference set.
\begin{center}
\fbox{%
\begin{minipage}{0.94\columnwidth}
\small
\textbf{Input:} normal training graphs $\mathcal{D}_{\rm tr}^0$, optional source references, query graph $G$.\\
\textbf{1. Structural language.} Compute LTDs, graph-level aggregates, motif/spectral summaries, and deterministic GraphDP strings.\\
\textbf{2. Frozen anchors.} Encode GraphDP with the instruction-aware text encoder and obtain native MRL slices by prefix truncation.\\
\textbf{3. Graph alignment.} Train the GIN encoder and graph-side projections on normal graphs using soft cosine alignment, Gram stabilization, and prototype shaping.\\
\textbf{4. Reference construction.} Use target normals for single-domain, source normals for zero-shot, or source plus $k$ target normals for few-shot.\\
\textbf{5. Spherical scoring.} Rank each query by the weighted angular $k$-NN reference-set score in Eq.~\eqref{eq:fixedscore}.
\end{minipage}}
\end{center}

\textbf{Observed transfer regimes.}
The same GraphDP fields can be transferable or insufficient depending on target semantics. Table~\ref{tab:regimes} summarizes the regimes reflected by Tables~\ref{tab:cross} and~\ref{tab:fewshot}.

\begin{table}[ht]
\centering
\caption{Observed transfer regimes and empirical evidence.}
\label{tab:regimes}
\resizebox{\columnwidth}{!}{%
\begin{tabular}{@{}lll@{}}
\toprule
Regime & Evidence & Interpretation \\
\midrule
Related-domain gain & Mol$\to$PROTEINS (+2.9), Mol$\to$D\&D (+4.6) & shared substructure vocabulary \\
Near-lossless coverage & Protein$\to$MUTAG (+0.4), Social/Protein$\to$AIDS (+0.06) & compatible normal geometry \\
Few-shot calibration & PROTEINS $65.59\to79.81$ from $0\to32$ shots & target references repair local density \\
Source dilution & All-domain helps PROTEINS but lowers MUTAG/IMDB-B & heterogeneous references broaden density \\
Hard target shift & REDDIT-B loses 15--20pp under transfer & different social semantics \\
Fine-grained gaps & ENZYMES and DHFR remain below the best baselines & missing biochemical substructure fields \\
\bottomrule
\end{tabular}}
\end{table}

\textbf{Frozen text anchors versus LoRA.}
The text-adaptation ablation starts from the same native-MRL graph encoder and tunes only LoRA adapters in the 0.6B text encoder. Table~\ref{tab:lora} shows that text adaptation is not a stable replacement for frozen anchors: it improves several graph-dominant datasets such as BZR, COX2, NCI1, and REDDIT-B, but degrades PROTEINS, D\&D, AIDS, IMDB-BINARY, and COLLAB, sometimes by a large margin. The warm-up column separates initialization quality from the tuned adaptation. The main method therefore keeps the instruction-aware text space fixed and trains graph-side alignment only.

\begin{table}[ht]
\centering
\caption{Frozen text anchors versus LoRA text adaptation under the default 0.6B setting (AUROC \%). The frozen column is the single-domain \ouralg.}
\label{tab:lora}
\resizebox{\columnwidth}{!}{%
\begin{tabular}{@{}lccccc@{}}
\toprule
Dataset & Frozen & LoRA warm-up & LoRA tuned & $\Delta$Warm-up & $\Delta$Frozen \\
\midrule
MUTAG & 91.42$\pm$0.87 & 90.02$\pm$1.41 & 90.31$\pm$1.22 & +0.29 & $-$1.11 \\
PROTEINS & 80.63$\pm$0.42 & 80.05$\pm$0.47 & 75.36$\pm$0.86 & $-$4.68 & $-$5.27 \\
D\&D & 79.41$\pm$0.49 & 76.23$\pm$1.56 & 76.24$\pm$1.84 & +0.01 & $-$3.17 \\
ENZYMES & 66.20$\pm$0.46 & 65.05$\pm$0.73 & 67.96$\pm$0.63 & +2.91 & +1.76 \\
DHFR & 69.38$\pm$0.37 & 68.25$\pm$0.29 & 68.42$\pm$0.26 & +0.17 & $-$0.96 \\
BZR & 80.93$\pm$0.38 & 81.33$\pm$0.20 & 83.51$\pm$0.22 & +2.19 & +2.58 \\
COX2 & 68.70$\pm$0.41 & 69.78$\pm$0.42 & 70.61$\pm$0.32 & +0.82 & +1.91 \\
AIDS & 99.32$\pm$0.59 & 99.67$\pm$0.06 & 89.13$\pm$2.40 & $-$10.54 & $-$10.19 \\
IMDB-B & 77.20$\pm$0.35 & 77.91$\pm$0.29 & 69.47$\pm$0.39 & $-$8.44 & $-$7.73 \\
NCI1 & 75.86$\pm$0.56 & 78.35$\pm$0.44 & 78.39$\pm$0.43 & +0.04 & +2.53 \\
COLLAB & 80.29$\pm$0.33 & 79.03$\pm$0.10 & 49.49$\pm$0.20 & $-$29.54 & $-$30.80 \\
REDDIT-B & 88.80$\pm$0.69 & 93.50$\pm$0.33 & 93.74$\pm$0.31 & $+$0.24 & +4.94 \\
\bottomrule
\end{tabular}}
\end{table}

\textbf{GraphDP and implementation summary.}
GraphDP is a deterministic key-value serialization. The compact field set is deliberately domain-neutral: it captures structural evidence that can be shared across molecular, protein, and social graphs without inserting class labels or anomaly status.

\begin{center}
\resizebox{\columnwidth}{!}{%
\begin{tabular}{@{}lll@{}}
\toprule
Group & Fields & Purpose \\
\midrule
Identity & graph domain and episode mode & scenario context \\
Size & nodes, edges, density, components & global scale and coverage \\
Degree & mean, std, entropy, quantiles & local connectivity profile \\
Clustering & local mean/std, transitivity & closure/community signal \\
Motifs & triangles, 4-cycles, motif density & local substructure signal \\
Core & maximum core number & centrality/robustness \\
Spectral & gap, Rayleigh quantiles & low-frequency connectivity \\
\bottomrule
\end{tabular}}
\end{center}

\begin{center}
\resizebox{\columnwidth}{!}{%
\begin{tabular}{@{}ll@{}}
\toprule
Component & Reported setting \\
\midrule
Graph encoder & 3-layer GIN, mean/max readout, spectral sketch \\
Graph code/slices & $D_g=2H+r+3=267$ ($H=128,r=8$), $D_s\in\{64,128,256,512\}$ \\
Text encoder & default Qwen3-Embedding-0.6B, 28 layers, 32K context, 1024-d output \\
Text slices & native MRL prefix truncation + $\ell_2$ normalization \\
Alignment loss & soft cosine over paired graph/text slices \\
Projection regularizer & dimension-adaptive Gram stabilization in Eq.~\eqref{eq:orth} \\
Prototype phase & $K=8$ vMF directions, EMA updates, compactness ramp \\
Primary scorer & GLASS angular $k$-NN-limit reference scorer \\
Few-shot adaptation & reference-set update only, no gradient updates \\
\bottomrule
\end{tabular}}
\end{center}

\textbf{Qwen3 embedding backbone.}
The default text encoder is Qwen3-Embedding-0.6B~\citep{qwen3embedding}. It is a dense Qwen3-based causal Transformer embedding model with 0.6B parameters, 28 layers, a 32K-token context window, 1024-dimensional embeddings, MRL support for flexible prefix dimensions, and instruction-aware inputs. In GLASS, the model is frozen and evaluated once per deterministic GraphDP string; only the graph encoder, graph-side projections, and prototype directions are trained. We follow the Qwen3 embedding query convention: an instruction is concatenated with the GraphDP query, and the final embedding is extracted from the last hidden state at the end-of-sequence token.

\textbf{Phase-aware GraphDP prompt examples.}
All text inputs share the same instruction and deterministic key-value fields. The phase token describes how the embedding will be used, not the graph label. Representative prompts are:
\begin{quote}\scriptsize\ttfamily\raggedright
\textbf{Single-domain/reference anchor}\quad
Instruct: Encode this graph description for graph-level anomaly detection.\\
Query: domain=protein; phase=single-reference; nodes=62; edges=145; density=0.076; components=1; degree\_q=2/4/6; degree\_entropy=1.71; clustering=0.18/0.07; motifs=triangles:38, fourcycles:91; core=max:5; spectral=gap:0.12, rayleigh\_q:0.08/0.31/0.66.\\[2pt]
\textbf{Zero-shot target query}\quad
Instruct: Encode this graph description for graph-level anomaly detection.\\
Query: domain=protein; phase=zero-shot-query; source\_domain=mol; nodes=74; edges=161; density=0.060; components=1; degree\_q=2/4/5; clustering=0.14/0.05; motifs=triangles:21, fourcycles:77; core=max:4; spectral=gap:0.09, rayleigh\_q:0.05/0.27/0.59.\\[2pt]
\textbf{Few-shot calibration reference}\quad
Instruct: Encode this graph description for graph-level anomaly detection.\\
Query: domain=protein; phase=few-shot-reference; trusted\_normal=true; nodes=69; edges=152; density=0.065; components=1; degree\_q=2/4/6; clustering=0.17/0.06; motifs=triangles:33, fourcycles:84; core=max:5; spectral=gap:0.11, rayleigh\_q:0.07/0.30/0.63.
\end{quote}
Anomaly status is never included. The field values above are illustrative; the actual prompts are generated deterministically from graph statistics.

\textbf{Complexity analysis.}
For a graph with $n$ nodes and $m$ edges, LTD extraction is $O(m)$ for degree and clustering-style statistics plus the cost of small-motif estimators; the lightweight spectral sketch with $r$ Lanczos/Chebyshev probes is approximately $O(rm)$. GraphDP serialization is linear in the number of retained fields and does not list nodes. Text encoding is a one-time cached cost per graph. With $N$ normal training graphs, $S$ slices, graph dimension $D_g$, and slice dimensions $D_s$, graph-side training costs $O(\sum_G L_{\rm GNN}m_GH+N\sum_sD_sD_g)$ per epoch, plus small prototype updates $O(NSKD_s)$. Inference for one query requires exact reference scoring $O(S|\mathcal{R}|D_s)$ after the graph/text forward pass; approximate cosine nearest-neighbor indexing can replace exact search for large reference sets. Zero-shot and few-shot deployment add no gradient-based adaptation: they only change $|\mathcal{R}|$.

\textbf{Runtime and memory considerations.}
All experiments were conducted on a server with 8 $\times$ NVIDIA RTX 5090 GPUs, dual Intel Xeon Gold 6530 CPUs (128 logical cores), and 512 GiB RAM. The main computational cost is separated into an offline text-embedding stage and an ordinary graph-side training stage. Let $T$ be the GraphDP token length, $L_{\rm txt}$ and $d_{\rm txt}$ the number of Transformer layers and hidden size of the text encoder, $L_{\rm GNN}$ the number of GNN layers, $H$ the GNN hidden size, $S$ the number of slices, and $D_s$ the slice dimensions. The text encoder forward pass costs approximately $O(L_{\rm txt}(T+d_{\rm txt})Td_{\rm txt})$ per graph, but all GraphDP embeddings are pre-computed and cached before graph-side training. Thus the per-epoch training cost is dominated by GNN message passing and graph-side projection/alignment. The final reference-set score is linear in the number of normal references and can be accelerated by approximate cosine nearest-neighbor indices in large deployments. What's more, during inference, increasing the zero-/few-shot reference set changes nearest-neighbor storage and lookup but introduces no additional trainable parameters and negligible computational complexity.

\begin{center}
\resizebox{\columnwidth}{!}{%
\begin{tabular}{@{}lll@{}}
\toprule
Step & Complexity & Comment \\
\midrule
LTD extraction & $O(m)$ + motif estimator & deterministic structural evidence \\
Spectral sketch & $O(rm)$ & $r$ Lanczos/Chebyshev probes \\
GraphDP text encoding & $O(L_{\rm txt}(T+d_{\rm txt})Td_{\rm txt})$ & one-time cached cost \\
Graph-side training & $O(\sum_G L_{\rm GNN}m_GH+N\sum_s D_sD_g)$ & per epoch \\
Prototype shaping & $O(NSKD_s)$ & small $K$ and $S$ \\
Reference scoring & $O(S|\mathcal{R}|D_s)$ & exact cosine $k$-NN \\
\bottomrule
\end{tabular}}
\end{center}

% BEGIN inlined privacy_appendix.tex
\FloatBarrier
\section{Private-Reference Reproducibility}
\label{app:private_protocol}
\subsection*{Source Learning and Record Construction}
The private-reference series uses a fresh public-source training entrypoint rather than the legacy evaluation-driven checkpoint path. It performs 150 fixed epochs of AdamW ($5\times10^{-5}$ learning rate, $10^{-4}$ weight decay, batch size 64), with cosine annealing to $10^{-6}$. The source-normal alignment warmup lasts 20 epochs. Thereafter, eight EMA vMF prototypes (momentum 0.999) and cached synthetic graph perturbations shape the source geometry; the prototype coefficient ramps to 0.3 over 20 epochs. No private target data update these states. The dimension-adaptive Gram penalty has coefficient 0.01 and averages across projection heads.

Canonical per-graph structural/attribute features are zero-padded to a fixed public width of 128; exceeding this schema bound raises an error rather than silently truncating attributes or adapting the model to target examples. Three GIN layers of hidden width 128, concatenated mean/max readout, and an 11-dimensional spectral code produce the 267-dimensional graph representation. Four native text prefixes supervise graph projections. The final model, including BatchNorm running statistics, is frozen before any target embeddings are extracted. Thus the target feature mapping is independent of the private reference set, as required by Proposition~\ref{prop:private_summary}.

The cached instruction is identical across graph records: ``Encode domain-invariant structural features of a graph,'' followed by instructions to emphasize topology, degree profiles, clustering, motifs, spectral gap, and community structure. This private series reuses the content-hashed cross-domain cache, not a new label- or phase-conditioned text model. A legacy normal-status field, where present in a cached description, is constant for both normal and anomalous queries and is not a per-query ground-truth label. No synthetic-negative status is embedded in the text anchors.

We use the minority class as the anomalous class (ties follow sorted class order), with 80\% of other-class graphs in the normal reference pool. Multiclass ENZYMES therefore treats the remaining five classes as normal. Each seed fixes the normal pool order and the held-out test set; smaller reference sets are nested prefixes. Exact source-record matches are excluded using a public-source hash lookup. The DP claim applies to the final eligible fixed-size reference collection, not to privately releasing eligibility counts or adapting $n$ from confidential metadata. Benchmark split counts and metrics are public experimental artifacts, outside a real deployment's private output interface.

\subsection*{Gaussian Calibration and Reuse}
Let $\Phi_{\rm N}$ denote the standard Gaussian CDF and $\Delta=2/n$. The mechanism uses the smallest noise scale satisfying
\begin{equation}
\begin{split}
&\Phi_{\rm N}\!\left(\frac{\Delta}{2\sigma_{\rm DP}}-
\frac{\epsilon\sigma_{\rm DP}}{\Delta}\right)\\[-2pt]
&\quad-e^\epsilon\Phi_{\rm N}\!\left(-\frac{\Delta}{2\sigma_{\rm DP}}-
\frac{\epsilon\sigma_{\rm DP}}{\Delta}\right)\le\delta,
\end{split}
\label{eq:dp_gaussian_calibration}
\end{equation}
implemented by monotone bisection~\citep{balle2018analytic}. For $\epsilon=4$, $\delta=10^{-5}$, this gives $\sigma_{\rm DP}=0.2703$ at $n=8$ and $0.0169$ at $n=128$. The public feature seed is fixed by channel, not optimized using test labels. Sine/cosine pairs ensure exact unit feature norm in real arithmetic. Numerical tests check the joint norm, replacement sensitivity, analytic calibration, kernel approximation, and the aggregate-noise law.

All private scores read only a released summary and public query features; no raw target-reference LOO normalization, reliability weights, or threshold fitting is performed. Multiple query evaluations of the same summary are postprocessing. In contrast, releasing additional summaries for the same confidential records requires composition; the separate budgets and ten draws in these public-data simulations are not a free multi-release private deployment.

The formal guarantee assumes the ideal Gaussian mechanism. Our benchmark implementation uses reproducible pseudorandom noise for scientific simulation, not a hardened production DP service. A private deployment must use an appropriate secure noise implementation, protect noise randomness from disclosure, enforce the release ledger, and avoid exposing reference embedding files. Public random-feature seeds may be shared; private noise seeds may not. Neither a small measured attack success rate nor a low utility loss replaces the mechanism's privacy proof.

\subsection*{Distributed Summary Simulation and Statistics}
For $J$ fixed disjoint sites of public sizes $n_j$, site $j$ adds noise with scale $2\gamma/n_j$, where $\gamma$ is the analytic Gaussian multiplier for a shared $(\epsilon,\delta)$. Aggregating with weights $n_j/n$ gives
\[
\operatorname{Var}(\eta_{\rm agg,\ell})
=\sum_{j=1}^{J}\left(\frac{n_j}{n}\frac{2\gamma}{n_j}\right)^2
=J\left(\frac{2\gamma}{n}\right)^2.
\]
Our simulation samples exactly this aggregate distribution. Because the statistic is linear, different fixed partitions with the same total count and site number yield the same aggregate distribution; an artificial IID/non-IID split table would add no evidence for this mechanism. Fixed record membership is essential: duplicate users across sites or data-dependent reassignment require a different privacy analysis. This is graph-record privacy with a trusted curator within each site, not local privacy for each incoming graph or client-level privacy for an entire institution.

We first average ten noise replicates within each model/split seed, then report the mean and sample standard deviation over four seeds. Within-model noise variability is separately retained in the released experiment summaries. Kernel/resolution diagnostics are all reported and never used to choose a dataset-specific main scorer. The full comparison below includes below-chance settings; score signs are not flipped after inspecting labels.

\subsection*{Approximation Costs and Difficult Settings}
\phantomsection\label{app:private_failures}
Across all 18 source--target settings, macro-average AUROC is 65.85 for target-reference 5-NN, 60.67 for exact finite-$\kappa$ KDE, 60.23 for the nonprivate random-feature summary, and 60.08 for its $\epsilon=4$ release. The noise-only difference is therefore 0.15 points, while the total difference from 5-NN is 5.77 points. The private extension is not a lossless implementation of exact nearest neighbors. The fixed global bandwidth is the largest source of aggregate degradation in this decomposition, although individual settings can benefit from smoothing.

Protein$\to$AIDS illustrates the distinction: exact 5-NN reaches 98.74, exact finite-bandwidth KDE 89.75, and the nonprivate summary 88.41. Its private average of 88.67 at $\epsilon=4$ fluctuates slightly above the nonprivate summary; this is not evidence of a systematic accuracy benefit from privacy noise. Small changes of this size should be interpreted alongside the retained within-model noise variability.

Fine-grained and shifted targets can remain difficult even without privacy noise. For example, the Mol$\to$DHFR target-only summary obtains 43.31 before noise and 43.36 at $\epsilon=4$; Mol$\to$ENZYMES obtains 49.70 and 49.58. These are failures of the representation/reference-density combination in this independently trained series, not an effect that can be repaired simply by increasing $\epsilon$. No score signs, source pools, or checkpoint epochs are selected retrospectively to improve these settings.

With one or a few private normals, both incomplete coverage and large $1/n$ noise produce low or variable scores. The public-source prior can reduce variance but retain a mismatched density. At $n=128$, source mixing lowers Mol$\to$PROTEINS from 73.78 to 72.79 and Protein$\to$AIDS from 88.67 to 85.55. Similarly, graph-only private scoring exceeds the fixed joint mixture on Protein$\to$AIDS. These outcomes motivate future public-validation-based source weighting or more local private density summaries; they are not used as a hidden modality or scorer selector in the present results.

\makeatletter
\setlength{\@dblfptop}{0pt}
\setlength{\@dblfpsep}{18pt}
\setlength{\@dblfpbot}{0pt plus 1fil}
\makeatother
% BEGIN inlined private_results/all_targets.tex
\begin{table*}[h!]
\centering
\footnotesize
\setlength{\tabcolsep}{3pt}
\renewcommand{\arraystretch}{1.08}
\caption{Private target-reference calibration, AUROC (\%). All columns use independently retrained public-source encoders. NN denotes angular mean 5-NN; Exact KDE and Sketch use $\kappa=4$. Private columns use the same target-only sketch, $\delta=10^{-5}$, and one total graph--text budget. Values are mean $\pm$ sample standard deviation over four model/split seeds, after averaging ten noise draws per private condition. No checkpoint or scorer is selected using target labels. All 18 disjoint source--target dataset settings are shown; MUTAG has 100 available reference graphs.}
\label{tab:dp_all}
\resizebox{\textwidth}{!}{%
\begin{tabular}{@{}llrccccccc@{}}
\toprule
Source & Target & $n$ & Source NN & Target NN & Exact KDE & Sketch & $\epsilon=1$ & $\epsilon=4$ & $\epsilon=8$ \\
\midrule
Mol & BZR & 128 & $39.42\!\pm\!5.51$ & $77.74\!\pm\!1.49$ & $62.76\!\pm\!2.10$ & $61.37\!\pm\!3.54$ & $59.03\!\pm\!3.07$ & $60.92\!\pm\!3.62$ & $61.03\!\pm\!3.46$ \\
Mol & COLLAB & 128 & $26.81\!\pm\!5.83$ & $59.94\!\pm\!11.26$ & $59.56\!\pm\!15.69$ & $55.09\!\pm\!17.16$ & $54.28\!\pm\!11.70$ & $55.14\!\pm\!15.41$ & $55.40\!\pm\!16.76$ \\
Mol & COX2 & 128 & $56.48\!\pm\!8.28$ & $59.61\!\pm\!3.21$ & $46.00\!\pm\!4.58$ & $45.77\!\pm\!4.14$ & $45.98\!\pm\!3.48$ & $45.75\!\pm\!4.30$ & $45.90\!\pm\!3.94$ \\
Mol & D\&D & 128 & $59.54\!\pm\!6.19$ & $71.74\!\pm\!2.48$ & $66.67\!\pm\!3.46$ & $66.86\!\pm\!3.27$ & $64.04\!\pm\!2.45$ & $66.82\!\pm\!3.34$ & $66.78\!\pm\!3.43$ \\
Mol & DHFR & 128 & $52.73\!\pm\!2.99$ & $51.80\!\pm\!4.98$ & $43.02\!\pm\!4.31$ & $43.31\!\pm\!4.48$ & $44.99\!\pm\!3.58$ & $43.36\!\pm\!4.50$ & $43.64\!\pm\!4.29$ \\
Mol & ENZYMES & 128 & $48.99\!\pm\!4.79$ & $49.11\!\pm\!3.45$ & $49.04\!\pm\!5.19$ & $49.70\!\pm\!4.86$ & $49.53\!\pm\!3.56$ & $49.58\!\pm\!4.78$ & $49.59\!\pm\!4.95$ \\
Mol & IMDB-B & 128 & $47.55\!\pm\!5.76$ & $57.85\!\pm\!2.87$ & $59.91\!\pm\!4.88$ & $59.93\!\pm\!4.89$ & $57.95\!\pm\!3.21$ & $59.70\!\pm\!4.54$ & $59.89\!\pm\!4.91$ \\
Mol & PROTEINS & 128 & $60.55\!\pm\!5.18$ & $72.34\!\pm\!1.85$ & $74.20\!\pm\!1.83$ & $74.03\!\pm\!1.93$ & $70.86\!\pm\!2.34$ & $73.78\!\pm\!1.87$ & $73.89\!\pm\!2.00$ \\
Mol & REDDIT-B & 128 & $65.51\!\pm\!8.19$ & $62.21\!\pm\!3.09$ & $59.21\!\pm\!1.43$ & $59.49\!\pm\!1.13$ & $58.87\!\pm\!1.87$ & $59.31\!\pm\!1.14$ & $59.40\!\pm\!0.96$ \\
Protein & AIDS & 128 & $29.69\!\pm\!14.76$ & $98.74\!\pm\!0.96$ & $89.75\!\pm\!5.83$ & $88.41\!\pm\!5.65$ & $82.31\!\pm\!4.50$ & $88.67\!\pm\!4.60$ & $88.32\!\pm\!5.54$ \\
Protein & BZR & 128 & $52.88\!\pm\!6.28$ & $70.78\!\pm\!1.98$ & $60.97\!\pm\!3.08$ & $58.76\!\pm\!3.65$ & $55.80\!\pm\!1.14$ & $58.75\!\pm\!3.15$ & $58.84\!\pm\!3.41$ \\
Protein & COLLAB & 128 & $22.86\!\pm\!3.34$ & $83.46\!\pm\!5.68$ & $77.09\!\pm\!2.85$ & $76.75\!\pm\!4.04$ & $72.28\!\pm\!3.94$ & $76.13\!\pm\!4.59$ & $76.56\!\pm\!4.12$ \\
Protein & COX2 & 128 & $48.69\!\pm\!10.89$ & $57.31\!\pm\!2.08$ & $49.91\!\pm\!2.98$ & $50.24\!\pm\!2.91$ & $51.55\!\pm\!1.21$ & $50.19\!\pm\!2.63$ & $50.54\!\pm\!2.72$ \\
Protein & DHFR & 128 & $52.30\!\pm\!4.23$ & $50.12\!\pm\!3.84$ & $47.55\!\pm\!2.06$ & $47.25\!\pm\!1.55$ & $48.13\!\pm\!2.91$ & $48.06\!\pm\!1.53$ & $47.49\!\pm\!1.58$ \\
Protein & IMDB-B & 128 & $64.29\!\pm\!4.95$ & $67.74\!\pm\!2.63$ & $61.21\!\pm\!7.67$ & $62.16\!\pm\!6.16$ & $60.62\!\pm\!4.76$ & $62.38\!\pm\!5.70$ & $62.15\!\pm\!5.98$ \\
Protein & MUTAG & 100 & $82.67\!\pm\!7.73$ & $80.84\!\pm\!5.00$ & $79.21\!\pm\!3.46$ & $77.35\!\pm\!2.98$ & $64.74\!\pm\!6.10$ & $75.66\!\pm\!4.03$ & $77.47\!\pm\!3.51$ \\
Protein & NCI1 & 128 & $58.31\!\pm\!0.89$ & $50.97\!\pm\!1.69$ & $49.04\!\pm\!3.92$ & $50.16\!\pm\!2.95$ & $50.05\!\pm\!3.14$ & $49.40\!\pm\!3.53$ & $49.99\!\pm\!2.99$ \\
Protein & REDDIT-B & 128 & $66.59\!\pm\!13.14$ & $62.91\!\pm\!6.27$ & $56.90\!\pm\!13.34$ & $57.46\!\pm\!13.78$ & $58.25\!\pm\!15.05$ & $57.80\!\pm\!14.36$ & $57.08\!\pm\!13.27$ \\
\bottomrule
\end{tabular}}
\end{table*}

% END inlined private_results/all_targets.tex

% BEGIN inlined private_results/combined_diagnostics.tex
\begin{table*}[t]
\begin{minipage}[t]{0.48\textwidth}
\centering\scriptsize
\setlength{\tabcolsep}{3pt}
\caption{Public representation and kernel settings: target-only DP AUROC at $n=128$, $\epsilon=4$. These diagnostics do not select the main configuration.}
\label{tab:dp_resolution}
\begin{tabular}{@{}lrrrr@{}}
\toprule
Setting & PROTEINS & AIDS & IMDB-B & Noise SD$^*$ \\
\midrule
64 only & 73.35 & 85.77 & 59.56 & 1.79 \\
128 only & 73.82 & 88.74 & 59.41 & 1.59 \\
64+128 (main) & 73.78 & 88.67 & 59.70 & 1.72 \\
256 only & 73.55 & 91.04 & 58.50 & 1.43 \\
All four slices & 73.82 & 89.51 & 59.26 & 1.64 \\
$M=64$ & 74.12 & 88.58 & 59.59 & 1.51 \\
$M=256$ & 73.87 & 89.68 & 59.86 & 1.61 \\
$\kappa=1$ & 72.80 & 88.04 & 60.12 & 2.21 \\
$\kappa=16$ & 69.24 & 90.00 & 59.84 & 1.49 \\
\bottomrule
\multicolumn{5}{l}{$^*$Mean within-model AUROC standard deviation over noise draws.}
\end{tabular}

\end{minipage}
\hfill
\begin{minipage}[t]{0.48\textwidth}
\centering\scriptsize
\setlength{\tabcolsep}{3pt}
\caption{Privacy metrics and distributed-noise simulation at $n=128$, $\epsilon=4$. AUPRC/FPR95 use the target-only joint sketch. Distributed AUROC uses the public-source mixture at fixed total $n$ and fixed disjoint site membership.}
\label{tab:dp_extra}
\begin{tabular}{@{}lrrrrr@{}}
\toprule
Target & AUPRC & FPR95 & $J=1$ & $J=2$ & $J=4$ \\
\midrule
PROTEINS & 89.99 & 77.37 & 72.79 & 72.56 & 72.20 \\
D\&D & 87.38 & 86.67 & 66.03 & 65.31 & 65.09 \\
AIDS & 87.27 & 28.45 & 85.55 & 85.33 & 84.77 \\
IMDB-B & 86.96 & 91.92 & 59.99 & 60.12 & 59.49 \\
\bottomrule
\end{tabular}

\end{minipage}
\end{table*}

% END inlined private_results/combined_diagnostics.tex

\FloatBarrier

% END inlined privacy_appendix.tex

\end{document}